\documentclass{article} 
\usepackage[final]{colm2025_conference}
\usepackage{microtype}
\usepackage{hyperref}
\usepackage{url}
\usepackage{booktabs}
\usepackage{lineno}
\usepackage{algorithm}
\usepackage{algpseudocode}
\usepackage{amsmath, amssymb, amsthm}
\usepackage{adjustbox}
\usepackage{graphicx}
\usepackage{multirow}
\usepackage{tikz}
\usepackage{svg}
\usepackage{tcolorbox}
\usetikzlibrary{arrows.meta, calc, shapes, positioning}
\usepackage{xcolor}
\definecolor{darkblue}{rgb}{0, 0, 0.5}
\hypersetup{colorlinks=true, citecolor=darkblue, linkcolor=darkblue, urlcolor=darkblue}

\newtheorem{theorem}{Theorem}
\newtheorem{lemma}{Lemma}
\newtheorem{corollary}{Corollary}

\title{Sculpting Subspaces: Constrained Full Fine-Tuning in LLMs for Continual Learning}

\author{\textbf{Nikhil Shivakumar Nayak}$^{1,4}$\thanks{Correspondence to: Nikhil Shivakumar Nayak \texttt{<nnayak@redhat.com>}.} , \textbf{Krishnateja Killamsetty}$^{2}$, \textbf{Ligong Han}$^{1,4}$,\\
\textbf{Abhishek Bhandwaldar}$^{2,4}$, \textbf{Prateek Chanda}$^{3}$, \textbf{Kai Xu}$^{1,4}$, \textbf{Hao Wang}$^{1,4}$,\\
\textbf{Aldo Pareja}$^{1,4}$, \textbf{Oleg Silkin}$^{1}$, \textbf{Mustafa Eyceoz}$^{1}$, \textbf{Akash Srivastava}$^{1,4}$ \\
\\
$^{1}$Red Hat AI Innovation \quad
$^{2}$IBM Research \quad
$^{3}$IIT Bombay \quad
$^{4}$MIT-IBM Watson AI Lab
}

\begin{document}

\ifcolmsubmission
\linenumbers
\fi

\maketitle

\begin{abstract}
Continual learning in large language models (LLMs) is prone to catastrophic forgetting, where adapting to new tasks significantly degrades performance on previously learned ones. Existing methods typically rely on low-rank, parameter-efficient updates that limit the model's expressivity and introduce additional parameters per task, leading to scalability issues. To address these limitations, we propose a novel continual full fine-tuning approach leveraging adaptive singular value decomposition (SVD). Our method dynamically identifies task-specific low-rank parameter subspaces and constrains updates to be orthogonal to critical directions associated with prior tasks, thus effectively minimizing interference without additional parameter overhead or storing previous task gradients. We evaluate our approach extensively on standard continual learning benchmarks using both encoder-decoder (T5-Large) and decoder-only (LLaMA-2 7B) models, spanning diverse tasks including classification, generation, and reasoning. Empirically, our method achieves state-of-the-art results—\textbf{up to 7\% higher} average accuracy than recent baselines like O-LoRA—and notably maintains the model’s general linguistic capabilities, instruction-following accuracy, and safety throughout the continual learning process by \textbf{reducing forgetting to near-negligible levels}. Our adaptive SVD framework effectively balances model plasticity and knowledge retention, providing a practical, theoretically grounded, and computationally scalable solution for continual learning scenarios in large language models.
\end{abstract}

\section{Introduction}
\label{sec:intro}
Language models have evolved into powerful general-purpose systems with remarkable capabilities across diverse tasks. From sentence classification and multilingual translation to complex reasoning and code generation, large language models (LLMs) such as GPT-3~\citep{brown2020language}, PaLM~\citep{chowdhery2022palm}, and LLaMA-2~\citep{touvron2023llama} have demonstrated unprecedented versatility. However, deploying these models in real-world enterprise scenarios presents a critical practical challenge: the necessity for \emph{continuous adaptation} to dynamically evolving data and emerging tasks without compromising previously acquired knowledge.

Consider scenarios where continual adaptation is crucial: an enterprise assistant continuously integrating new company products, updated policies, and emerging customer needs; or a medical language model assimilating the latest research findings, novel treatment protocols, and evolving medical terminology. Continuously retraining large language models \emph{from scratch} with all accumulated data each time new tasks or datasets arrive is computationally prohibitive and unsustainable at scale.

\emph{Continual learning} addresses this challenge by enabling models to learn sequentially from data streams. However, large language models are particularly prone to \emph{catastrophic forgetting}~\citep{mccloskey1989catastrophic, kirkpatrick2017overcoming}, a phenomenon where adapting to new tasks significantly degrades performance on previously mastered ones. This is due to the interdependent nature of distributed representations in neural networks, where beneficial updates for a new task interfere with critical knowledge for prior tasks.

\begin{figure}
  \centering
\includegraphics[width=0.95\textwidth]{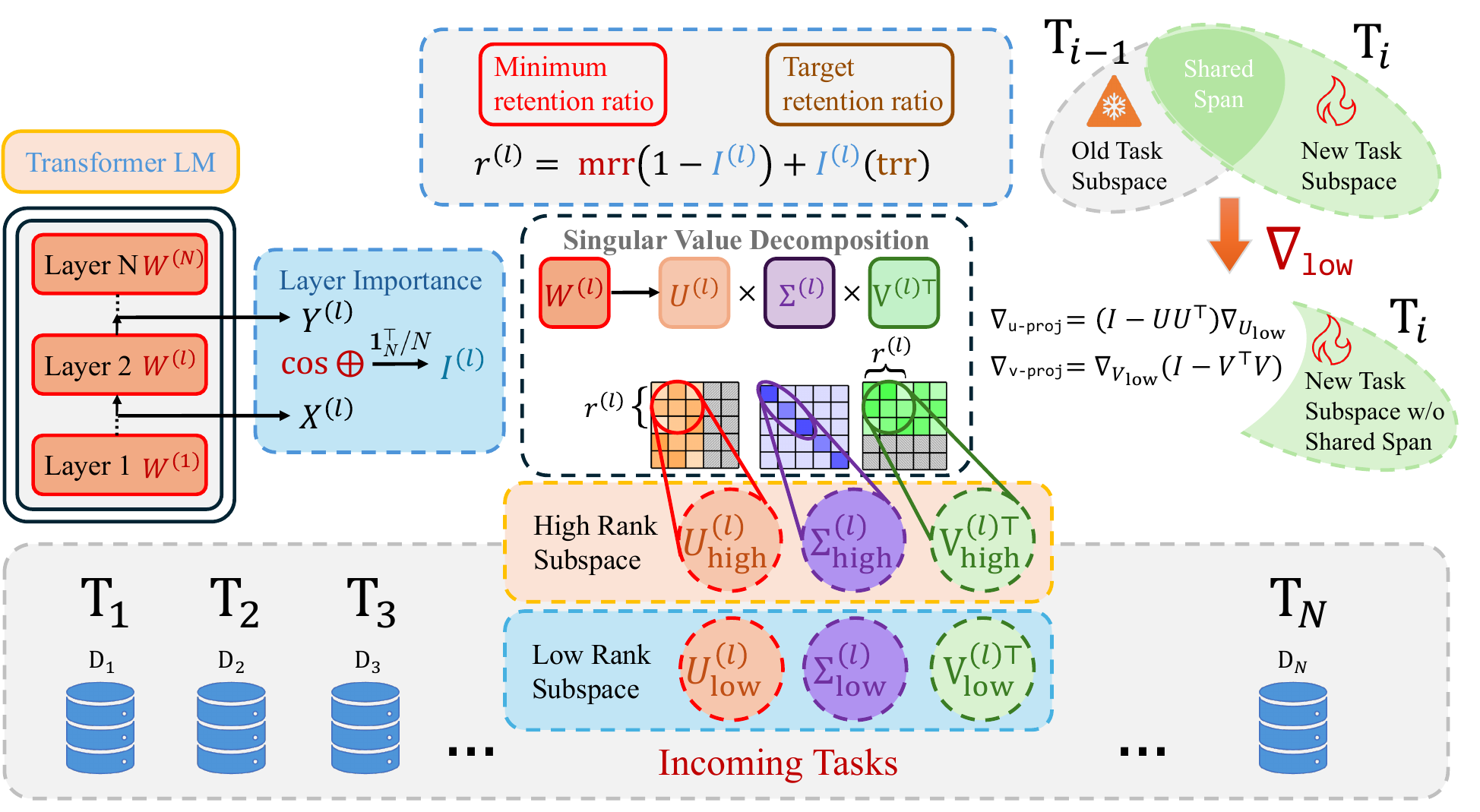}
  \caption{\textbf{Overview of our Adaptive SVD-based Continual Fine-tuning Method.} For each parameter matrix in the network, we perform SVD decomposition to identify high-rank components (associated with larger singular values) that encode crucial knowledge from previous tasks, and low-rank components (associated with smaller singular values) that contribute minimally to model performance. When learning a new task, gradient updates are projected onto the low-rank subspace orthogonal to previous task representations, allowing full parameter updates while minimizing catastrophic forgetting.}
      \label{fig:method_overview}
\end{figure}


Existing continual learning methods for LLMs primarily rely on parameter-efficient fine-tuning techniques. Approaches utilizing Adapters~\citep{houlsby2019parameter} or Low-Rank Adaptation (LoRA)~\citep{hu2022lora} selectively update small subsets of parameters while freezing the majority of the network. Although these methods mitigate forgetting to some extent, interference can still persist. More recent techniques like Orthogonal LoRA (O-LoRA)~\citep{wang2024olora} and Interference-free LoRA (InfLoRA)~\citep{liang2024inflora} add orthogonality constraints to further reduce task interference. However, these approaches face fundamental limitations: (1) they constrain the model's expressive capacity by restricting updates to small parameter subspaces, (2) they require additional parameters for each new task, increasing memory footprint and inference complexity, and (3) they necessitate task-specific architectures, complicating deployment in real-world settings.

Alternatively, model merging techniques~\citep{ilharco2022patching, yadav2023tiesmerging} fine-tune models separately for each task and subsequently combine them. While this approach can effectively preserve task-specific knowledge to a certain extent, it requires maintaining multiple full-model copies during training and significant expertise to achieve strong performance. Moreover, it often struggles to match the performance of models jointly trained on all tasks. This raises a fundamental research question:

\begin{center}
\emph{How can we enable LLMs to continuously learn new tasks without compromising previously acquired knowledge, while maintaining full model expressivity and avoiding parameter growth?} 
\end{center}

In this work, we introduce a novel continual learning approach that utilizes \emph{adaptive low-rank subspace updates} guided by singular value decomposition (SVD). Our method is built upon a key insight supported by recent research~\citep{sharma2023truth}: neural network weight matrices often contain significant redundancy, with many parameter directions (particularly those associated with smaller singular values) contributing minimally to overall performance. We capitalize on this observation by dynamically identifying these underutilized directions and repurposing them for learning new tasks, while preserving crucial directions that encode knowledge from previously learned tasks.

Specifically, we perform an adaptive SVD-based decomposition for each weight matrix, isolating high-rank components (larger singular values) encoding essential past knowledge and low-rank components (smaller singular values) suitable for learning new tasks. Gradient updates for new tasks are constrained to these low-rank subspaces orthogonal to previously learned task representations, enabling effective full-parameter updates without forgetting prior knowledge. Importantly, unlike parameter-efficient methods, our approach maintains a fixed parameter count regardless of the task sequence length and fully leverages the model's expressive capacity. Figure~\ref{fig:method_overview} provides a visual overview of our method.

Our contributions can be summarized as follows:

\noindent \textbf{1. A geometric approach for continual learning:} We propose a theoretically grounded method that leverages the geometric properties of weight matrices—via adaptive SVD—to identify and reuse parameter subspaces with minimal interference on previously learned tasks. This effectively balances the plasticity needed for new tasks with stability for retaining prior knowledge.

\noindent \textbf{2. Full-model fine-tuning without extra memory:} Our method updates \emph{all} parameters while maintaining a fixed footprint, avoiding new modules or stored gradients for each task and thus scaling gracefully to many tasks.

\noindent \textbf{3. State-of-the-art performance on diverse tasks:} We demonstrate consistent gains across classification, generation, and reasoning benchmarks using T5-Large and LLaMA-2 (7B). Compared to existing methods, our approach achieves \emph{better accuracy, stronger knowledge retention, and nearly negligible forgetting}—while preserving general linguistic capabilities, instruction-following, and safety.

\noindent \textbf{4. Thorough empirical and theoretical validation:} We provide in-depth analyses verifying the effective repurposability of low-rank subspaces, showing that these directions can be used for new tasks without degrading old ones. Our experiments (Sections~\ref{subsec:assumption_validation}, \ref{subsec:low_rank_assumption}) confirm practical robustness while we evaluate theoretical soundness in Appendix~\ref{appendix:theorem_tight_bounds}.

The remainder of this paper is structured as follows. Section~\ref{sec:related_work} reviews relevant continual learning methods. Section~\ref{sec:methodology} introduces our adaptive subspace-based fine-tuning method with theoretical justifications. Section~\ref{sec:experimental_results} describes the benchmarks and evaluation metrics along with experimental results. Finally, Section~\ref{sec:conclusion} summarizes our contributions and outlines future research directions.

\paragraph{Code Availability.} 
Our code implementation is available at \url{https://github.com/Red-Hat-AI-Innovation-Team/orthogonal-subspace-learning}.

\section{Related Work}
\label{sec:related_work}
Continual learning methods for large language models primarily tackle catastrophic forgetting \citep{kirkpatrick2017overcoming, zenke2017continual} and generally fall into three main categories: \textit{parameter-efficient fine-tuning}, \textit{regularization and isolation approaches}, and \textit{unconstrained full-model fine-tuning and merging techniques}.

\noindent \textbf{Parameter-Efficient Fine-Tuning: } Parameter-efficient approaches address catastrophic forgetting by freezing most pretrained parameters and updating only a small subset of task-specific parameters. Prominent examples include Adapter modules \citep{houlsby2019parameter} and various Low-Rank Adaptation (LoRA) methods such as O-LoRA \citep{wang2024olora} and InfLoRA \citep{liang2024inflora}. These techniques effectively reduce interference by isolating updates within small, constrained subspaces. However, they limit model expressiveness due to the restricted update space and often require additional parameters per task, raising scalability concerns.

\noindent \textbf{Regularization and Isolation Approaches: } Regularization-based methods such as Elastic Weight Consolidation (EWC) \citep{kirkpatrick2017overcoming} and Synaptic Intelligence (SI) \citep{zenke2017continual} penalize updates to important parameters without completely preventing them. While these approaches allow for full-model updates, they do not fundamentally eliminate interference, causing gradual performance degradation across multiple tasks. In contrast, parameter isolation techniques, such as PackNet \citep{mallya2018packnet} and Progressive Neural Networks \citep{rusu2016progressive}, maintain separate parameter subsets or modules for each task. These approaches effectively prevent interference but introduce redundancy and face scalability challenges as the number of tasks increases.

\noindent \textbf{Full-Model Fine-Tuning and Model Merging: } Standard full-model fine-tuning methods update all parameters when learning each new task, fully exploiting the model's expressive power but risking severe catastrophic forgetting due to conflicting updates \citep{luo2025empirical}. On the other hand, model merging approaches, such as PATCHING \citep{ilharco2022patching}, TIES \citep{yadav2023tiesmerging}, represent an alternative strategy where models are fine-tuned separately for each task and subsequently combined into a unified multitask model by resolving parameter conflicts post-hoc. While effective, these methods incur higher computational costs due to multiple rounds of training and merging.

\noindent \textbf{Positioning Our Work: } Our approach introduces a novel constrained full-parameter update method that differs fundamentally from existing categories. Unlike parameter-efficient approaches, we leverage the entire parameter space, maximizing expressive capacity. Unlike isolation approaches, we don't partition parameters or require additional task-specific modules. Unlike constrained full fine-tuning, we explicitly mitigate interference through geometric constraints. Specifically, we dynamically identify low-rank subspaces via Singular Value Decomposition (SVD) and constrain updates to be orthogonal to previously learned task representations. This geometric approach to interference minimization ensures knowledge preservation while maintaining update flexibility. By operating in the full parameter space while enforcing orthogonality constraints, our method achieves a unique balance between knowledge retention and model plasticity, providing a theoretically grounded and practically scalable solution for continual learning in large language models.

\section{Methodology}
\label{sec:methodology}

Our approach addresses continual learning in large language models by leveraging adaptive low-rank updates guided by Singular Value Decomposition (SVD). We strategically preserve critical knowledge from previous tasks by constraining parameter updates away from dominant (high-rank) singular directions, while enabling model adaptation within complementary (low-rank) directions.

\subsection{Problem Setup and Notation}
\label{subsec:problem_setup}

Let the parameters of an LLM be denoted as:
$$
\theta = \{\mathbf{W}^{(1)}, \mathbf{W}^{(2)}, \dots, \mathbf{W}^{(L)}\},
$$
where each $\mathbf{W}^{(l)} \in \mathbb{R}^{d_{O}^{(l)} \times d_{I}^{(l)}}$ represents the weight matrix of layer $l$. Practical deployments involve matrices with millions or billions of parameters, underscoring the necessity of efficient continual updates.

Given sequential tasks $\{\mathcal{D}_1, \mathcal{D}_2, \dots, \mathcal{D}_T\}$, each defined by data pairs $\{(x_i^t, y_i^t)\}_{i=1}^{n_t}$, our goal is to sequentially adapt parameters $\theta$ to task $\mathcal{D}_t$ without significant performance degradation on previously learned tasks $\mathcal{D}_1, \dots, \mathcal{D}_{t-1}$. Training repeatedly from scratch is computationally prohibitive, necessitating efficient incremental updates.

\subsection{Low-Rank and High-Rank Subspaces via SVD}
\label{subsec:svd_idea}
Extensive empirical evidence shows neural network parameters possess substantial redundancy~\citep{sharma2023truth, hartford2024spectrumtargetedtrainingsignal}, where directions associated with small singular values minimally impact critical model knowledge. Conversely, larger singular values typically encapsulate vital knowledge. Leveraging this observation, we propose:

\begin{quote}
\emph{Projecting parameter updates away from high singular-value directions, preserving previously acquired knowledge, and utilizing low singular-value directions for adaptation to new tasks.}
\end{quote}

Formally, we perform Singular Value Decomposition (SVD) on each weight matrix $\mathbf{W}^{(l)}$ at layer $l$:
\begin{equation}
\label{eq:svd_decomposition}
\mathbf{W}^{(l)} = \mathbf{U}^{(l)} \Sigma^{(l)} (\mathbf{V}^{(l)})^\top,
\end{equation}
where singular values in $\Sigma^{(l)}$ are sorted in descending order. We compute this decomposition once per task, adding minimal overhead compared to full model training.

\subsection{Determining Layer Importance via Input--Output Similarity}
\label{subsec:layer_importance}

Inspired by AdaSVD~\citep{li2025adasvdadaptivesingularvalue}, we quantify layer importance using cosine similarity between a layer's input activations $\mathbf{X}^{(l)}$ and its linear outputs $\mathbf{Y}^{(l)} = \mathbf{W}^{(l)}\mathbf{X}^{(l)}$. Specifically, when evaluating layer importance for task $t+1$, we compute the similarity using data samples from the previous task $t$ as follows:

\begin{equation}
\label{eq:layer_importance}
I^{(l)} = \frac{1}{N} \sum_{i=1}^{N}\text{cosine\_similarity}(\mathbf{X}^{(l)}_i, \mathbf{Y}^{(l)}_i)
\end{equation}

where $N$ denotes the number of data samples from task $t$. Higher similarity indicates minimal directional change, signifying that the layer predominantly preserves rather than transforms activation representations. Such layers are essential for retaining features and ensuring stable propagation of information across tasks. Importance scores are also normalized to have an average of one across layers: $\frac{1}{L}\sum_{l=1}^{L} I^{(l)} = 1.$

\subsection{Adaptive Rank Selection}
\label{subsec:adaptive_rank}
Given the importance of the layer $I^{(l)}$, we introduce two hyperparameters controlling the retention of singular vectors:
\begin{itemize}
    \item \textbf{Minimum Retention Ratio (mrr)}, ensuring minimal essential retention even for the least critical layers.
    \item \textbf{Target Retention Ratio (trr)}, defining the upper retention bound for highly critical layers.
\end{itemize}

The fraction of singular vectors preserved at each layer is computed as:
\begin{equation}
\label{eq:layer_fraction}
r^{(l)} = \mathrm{mrr} + I^{(l)}(\mathrm{trr} - \mathrm{mrr}).
\end{equation}
Singular vectors are partitioned into high-rank $(\mathbf{U}^{(l)}_{\text{high}}, \mathbf{V}^{(l)}_{\text{high}})$ and low-rank $(\mathbf{U}^{(l)}_{\text{low}}, \mathbf{V}^{(l)}_{\text{low}})$ subspaces accordingly; implementation-specific values are provided in Appendix~\ref{appendix:implementation_details}.

\subsection{Orthogonal Gradient Updates in Low-Rank Subspace}
\label{subsec:orthogonal_updates}

To minimize catastrophic forgetting, we enforce updates within the low-rank subspace orthogonal to the high-rank directions:
\begin{equation}
\label{eq:orthogonal_projection}
\nabla \mathbf{W}^{(l)}_{\mathrm{proj}} = \nabla \mathbf{W}^{(l)} - \mathbf{U}^{(l)}_{\text{high}}(\mathbf{U}^{(l)}_{\text{high}})^\top \nabla \mathbf{W}^{(l)} \mathbf{V}^{(l)}_{\text{high}}(\mathbf{V}^{(l)}_{\text{high}})^\top.
\end{equation}

This ensures updates do not overwrite knowledge encoded in critical parameter directions, promoting knowledge retention while enabling effective adaptation.

\subsection{Algorithm Summary}
\label{subsec:method_summary}

\begin{algorithm}[ht]
\caption{Adaptive Low-Rank Continual Learning via SVD}\label{alg:adaptive_svd}
\begin{algorithmic}[1]
\Require Initial parameters \(\theta = \{\mathbf{W}^{(l)}\}_{l=1}^{L}\), tasks \(\{\mathcal{D}_t\}_{t=1}^{T}\), hyperparameters \(\mathrm{mrr}, \mathrm{trr}\).

\For{task \(t = 1, \dots, T\)}
    \State Compute importance \(I^{(l)}\) from layer activations (Eq.~\eqref{eq:layer_importance}); normalize across layers.
    \For{layer \(l = 1, \dots, L\)}
        \State Compute SVD: \(\mathbf{W}^{(l)} = \mathbf{U}^{(l)} \Sigma^{(l)} (\mathbf{V}^{(l)})^\top\).
        \State Retain top \(r^{(l)} = \mathrm{mrr} + I^{(l)}(\mathrm{trr}-\mathrm{mrr})\) singular vectors.
    \EndFor
    
    \While{not converged on task \(\mathcal{D}_t\)}
        \State Sample mini-batch, compute loss \(\mathcal{L}_t(\theta)\), gradients \(\nabla \mathbf{W}^{(l)}\).
        \State Project gradients onto low-rank subspace via:
        \[
        \nabla \mathbf{W}^{(l)}_{\mathrm{proj}} = \nabla \mathbf{W}^{(l)} - \mathbf{U}^{(l)}_{\text{high}}(\mathbf{U}^{(l)}_{\text{high}})^\top \nabla \mathbf{W}^{(l)} \mathbf{V}^{(l)}_{\text{high}}(\mathbf{V}^{(l)}_{\text{high}})^\top
        \]
        \State Update parameters with projected gradients.
    \EndWhile
\EndFor

\Ensure Parameters \(\theta\) updated continually without significant forgetting.
\end{algorithmic}
\end{algorithm}

Our adaptive low-rank continual learning procedure is summarized in Algorithm~\ref{alg:adaptive_svd}.



\subsection{Exploration of Alternative Rank Approximation Methods}
\label{subsec:alternative_methods}

Before developing our adaptive method, we explored rank approximation approaches including LASER~\citep{sharma2023truth} and SPECTRUM~\citep{hartford2024spectrumtargetedtrainingsignal}:

\begin{itemize}
    \item LASER's fixed-rank strategy fails to reflect layer-wise variability, resulting in suboptimal retention–adaptation trade-offs.
    \item SPECTRUM's random-matrix thresholding (using Marchenko–Pastur distribution) is unstable under sequential tasks with diverse distributions. Refer to Appendix~\ref{subsec:singular_value_analysis} for rank approximation results with random-matrix thresholding.
    \item Neither explicitly enforces orthogonality constraints crucial for continual learning.
\end{itemize}

These limitations motivated our adaptive, orthogonality-constrained subspace partitioning method based on explicit layer importance.

\subsection{Theoretical Justification of Adaptive Rank Selection}
\label{subsec:theoretical_justification}

We rigorously justify our adaptive rank selection method through a formal theoretical analysis using a second-order Taylor expansion of the task-specific loss landscape, detailed in Appendix~\ref{appendix:theorem_tight_bounds}. This analysis explicitly demonstrates that preserving parameter directions associated with the highest Hessian eigenvalues—representing directions of greatest curvature—effectively minimizes catastrophic forgetting. Ideally, one would restrict parameter updates away from these high-curvature subspaces, enabling safe updates along lower-curvature directions.

However, explicitly computing and decomposing the Hessian is computationally prohibitive for large-scale language models. Therefore, we employ an efficient approximation inspired by empirical evidence from \citet{Haink2023HessianEA}, who show a robust correlation between the Hessian’s largest eigenvalues and the largest singular values of the model's weight matrices. Leveraging this insight, we replace the expensive Hessian decomposition with Singular Value Decomposition (SVD) on the weight matrices. By retaining the top singular vectors—corresponding to critical knowledge learned from previous tasks—we effectively approximate freezing the high-curvature Hessian directions. Simultaneously, we allow updates within the subspace defined by lower singular values, thereby efficiently enabling adaptation to new tasks without substantial forgetting.

Further supporting our approach, empirical findings \citep{sharma2023truth, li2025adasvdadaptivesingularvalue} highlight that layers with higher input-output similarity exhibit significantly greater Hessian curvature. Our adaptive layer-wise rank allocation strategically exploits this property: layers identified as crucial (high input-output similarity) receive greater singular vector retention, thereby preserving essential knowledge. Conversely, less critical layers allow more aggressive updates in the low-curvature subspace. This layer-specific adaptive strategy aligns well with the theoretical framework, resulting in superior performance in practice.

In Section~\ref{sec:experimental_results}, we empirically validate that our adaptive, SVD-based rank selection method significantly reduces forgetting and consistently outperforms both naive full fine-tuning and uniform low-rank projection baselines, effectively bridging the theoretical ideal with a practical, scalable solution.

\subsection{Validation of Low-Rank Subspace Assumptions}
\label{subsec:assumption_validation}

Our approach assumes that lower singular vectors can safely accommodate new knowledge without significant forgetting. We empirically validate this by systematically pruning low singular value vectors on pre-trained models. Our experiments confirm a negligible performance drop when removing substantial fractions of lower singular vectors (see Section~\ref{subsec:low_rank_assumption}). This supports the theoretical redundancy hypotheses~\citep{chen2020lottery, sharma2023truth}, validating our adaptive low-rank continual learning strategy.

\section{Experimental Results}
\label{sec:experimental_results}
We comprehensively evaluate our adaptive SVD-based continual learning method on established continual learning benchmarks, comparing it extensively with recent state-of-the-art (SOTA) baselines, notably O-LoRA~\cite{wang2024olora}. Our experiments aim to demonstrate the effectiveness, scalability, and practicality of our approach in realistic continual learning scenarios.

\subsection{Benchmarks and Evaluation Protocol}
We adopt two widely-used benchmarks reflecting varying levels of complexity and task diversity:

\textbf{Standard Continual Learning Benchmark (5 Tasks)} introduced by~\citet{Zhang2015CharacterlevelCN}, consisting of classification tasks: AG News, Amazon Reviews, Yelp Reviews, DBpedia, and Yahoo Answers.

\textbf{Extended Continual Learning Benchmark (15 Tasks)}, introduced by~\citet{razdaibiedina2023progressive}, combining tasks from multiple sources, including GLUE (MNLI, QQP, RTE, SST-2), SuperGLUE (WiC, CB, COPA, MultiRC, BoolQ), and IMDB, along with the original 5-task benchmark.

We evaluate two popular large language model architectures, T5-Large (encoder-decoder) and LLaMA-2 7B (decoder-only), using the widely-adopted metric of Average Accuracy (AA), computed across all tasks after training on the final task. To ensure robustness, we follow standard protocols, averaging results over three independent runs with randomly permuted task sequences. Implementation details, hardware configurations, and training hyperparameters for both T5-Large and LLaMA-2 7B models are provided in Appendix~\ref{appendix:implementation_details}.

\subsection{Baseline Methods}

We position our adaptive SVD approach clearly against representative continual learning paradigms:
\begin{itemize}
    \item \textbf{Sequential full-model fine-tuning (SeqFT)}: serves as a lower-bound baseline, prone to catastrophic forgetting.
    \item \textbf{Parameter-efficient LoRA variants} including SeqLoRA, IncLoRA, and the recent SOTA, O-LoRA~\cite{wang2024olora}, which utilize low-rank adapters.
    \item \textbf{Replay-based approaches}, such as standard replay buffers.
    \item \textbf{Regularization methods}, including Elastic Weight Consolidation (EWC)~\cite{kirkpatrick2017overcoming} and Learning without Forgetting (LwF)~\cite{li2017learning}.
    \item \textbf{Prompt-based techniques}, including L2P~\cite{wang2022learning} and ProgPrompt~\cite{razdaibiedina2023progressive}.
    \item \textbf{PerTaskFT}: trains a separate model per task, offering strong performance but requiring extensive computational resources and storage.
    \item \textbf{Multi-task Learning (MTL)}: trains a single model simultaneously on all tasks, representing an ideal upper bound by relaxing continual learning constraints.
\end{itemize}

\subsection{Main Results}

\begin{table}[h]
\centering
\caption{Comparison of Average Accuracy (\%) across standard continual learning benchmarks}
\label{tab:results_summary}
\begin{tabular}{lcc}
\toprule
\textbf{Method} & \textbf{T5-Large (5 tasks)} & \textbf{T5-Large (15 tasks)} \\
\midrule
SeqFT & 28.5 & 7.4 \\
SeqLoRA & 43.7 & 1.6 \\
IncLoRA & 66.4 & 61.2 \\
Replay & 57.8 & 54.2 \\
EWC & 48.7 & 45.1 \\
LwF & 52.3 & 46.9 \\
L2P & 60.7 & 56.1 \\
LFPT5 & 72.7 & 69.2 \\
O-LoRA & 75.8 & 69.6 \\
\textbf{Ours (Adaptive SVD)} & \textbf{75.9} & \textbf{71.3} \\
\midrule
SLERP (Full Model Merge) & 43.1 & 2.2 \\
TIES (LoRA Adapter Merge) & 37.1 & 6.9 \\
\midrule
ProgPrompt & 75.1 & 77.9 \\
PerTaskFT & 70.0 & 78.1 \\
MTL (Upper Bound) & 80.0 & 76.5 \\
\bottomrule
\end{tabular}
\end{table}

Table~\ref{tab:results_summary} clearly shows that our adaptive SVD approach outperforms or matches all baselines on both 5-task and 15-task benchmarks. Importantly, compared to O-LoRA—the current SOTA parameter-efficient baseline—our method achieves superior accuracy, particularly in the more challenging 15-task scenario (71.3\% vs. 69.6\%), highlighting its effectiveness in maintaining task knowledge over extended task sequences. Notably, while PerTaskFT achieves high performance, it requires training separate models per task, making it computationally impractical. MTL represents an idealized scenario, training on all tasks simultaneously, thus serving as an upper-bound performance indicator. A comparison with model merging methods, SLERP and TIES, is provided in Appendix~\ref{appendix:model_merging}, with corresponding results included in Table~\ref{tab:results_summary}.

\subsection{Performance on the TRACE Benchmark}
To further illustrate our method's capability in more realistic continual learning environments, we evaluate it on TRACE~\cite{wang2023trace}, which includes diverse and challenging instruction-tuned tasks across multilingual understanding, domain-specific knowledge, arithmetic reasoning, and coding.

\begin{table}[h]
\centering
\caption{TRACE benchmark performance using LLaMA-2-7B-Chat.}
\label{tab:trace_results}
\resizebox{0.8\textwidth}{!}{%
\begin{tabular}{lcc}
\toprule
\textbf{Method} & \textbf{Average Accuracy (\%)} & \textbf{Backward Transfer (\%)} \\
\midrule
SeqFT & 23.0 & -8.3 \\
O-LoRA & 41.3 & 6.2 \\
\textbf{Ours (Adaptive SVD)} & \textbf{48.4} & \textbf{7.1} \\
\midrule
PerTaskFT & 57.6 & NA \\
MTL & 52.3 & NA \\
\bottomrule
\end{tabular}}
\end{table}

Results in Table~\ref{tab:trace_results} emphasize our method's ability to effectively retain and transfer knowledge across tasks. Our approach achieves notably higher average accuracy and backward transfer compared to O-LoRA, demonstrating superior robustness to forgetting, critical for practical deployments.


\begin{table}[h]
  \centering
  \renewcommand{\arraystretch}{0.95}
  \caption{Comparison of general ability scores across six diverse evaluation tasks between the base LLaMA-2-7B chat model and our adaptive SVD-based continual learner.}
  \label{tab:general_ability_comparison}
  \begin{tabular}{lcccccc}
    \toprule
    \textbf{Model} & \textbf{MMLU} & \textbf{GSM} & \textbf{BBH} & \textbf{TydiQA} & \textbf{BoolQA} & \textbf{PIQA} \\
    \midrule
    Base Instruct Model 
      & 46.6 & \textbf{26.1} & \textbf{40.2} & 23.5 & 70.5 & 76.2 \\
    Ours (Adaptive SVD) 
      & \textbf{47.7} & 7.7 & 34.2 & \textbf{35.8} & \textbf{76.6} & \textbf{77.6} \\
    \bottomrule
  \end{tabular}
\end{table}

\textbf{Retention of General Capabilities and Safety.} We explicitly evaluate the preservation of general abilities, instruction-following, and safety after continual learning using benchmarks proposed by TRACE. Table~\ref{tab:general_ability_comparison} illustrates our method's effectiveness in preserving or enhancing core language capabilities compared to the original instruction-tuned model. Our approach retains multilingual comprehension and reasoning abilities exceptionally well, a key differentiator for real-world applicability. Table~\ref{tab:trace_help_safety_breakdown} demonstrates that our approach retains superior instruction-following ability and safety performance compared to baselines.

\begin{table}[htbp]
\centering
\caption{Win / Tie / Lose breakdown (\%) for instruction-following and safety evaluations against the LLaMA-2-7B-Chat base model.}
\label{tab:trace_help_safety_breakdown}
\begin{tabular}{lccc|ccc}
\toprule
\multirow{2}{*}{\textbf{Method}} & \multicolumn{3}{c|}{\textbf{Instruction (Helpfulness)}} & \multicolumn{3}{c}{\textbf{Safety}} \\
 & Win & Tie & Lose & Win & Tie & Lose \\
\midrule
Replay & 10 & 18 & 72 & 0 & 88 & 12 \\
LoRASeqFT & 3 & 4 & 94 & 0 & 86 & 14 \\
SeqFT & 14 & 34 & 53 & 0 & 98 & 2 \\
\textbf{Ours (Adaptive SVD)} & \textbf{24} & \textbf{56} & \textbf{20} & \textbf{18} & \textbf{78} & \textbf{4} \\
\bottomrule
\end{tabular}
\end{table}

\section{Conclusion}
\label{sec:conclusion}
\vspace{-2mm}
As large language models (LLMs) become increasingly central to real-world applications, continually adapting them without erasing prior knowledge is essential. We presented a novel continual learning framework that uses adaptive singular value decomposition (SVD) to isolate low-rank subspaces for new tasks while preserving critical directions for previously acquired knowledge. Unlike parameter-efficient techniques that freeze most weights or add modules per task, our method operates on \emph{all} model parameters with fixed memory, preventing catastrophic forgetting through orthogonal subspace updates. Extensive empirical evaluations demonstrate our method's effectiveness across diverse benchmarks: (1) \emph{On the 5-task benchmark with LLaMA-2 7B}, we achieved \textbf{79.6\%} accuracy, surpassing the current SOTA by over 3 percentage points; (2) \emph{or the challenging 15-task sequence with T5-Large}, we reached \textbf{71.3\%} accuracy, outperforming all parameter-efficient competitors; (3) \emph{On the realistic TRACE benchmark with LLaMA-2 7B-Chat}, our method attained \textbf{48.4\%} average accuracy without requiring simultaneous multi-task access or multiple specialized models. Crucially, our approach preserved general capabilities, instruction-following behavior, and safety throughout continual learning—essential properties for deployment in production environments. Our adaptive SVD method provides a mathematically principled solution to the fundamental tension between stability and plasticity in neural networks, offering a scalable path toward continuously evolving language models that efficiently accumulate knowledge without forgetting. By demonstrating that full parameter updates can be performed without compromising previously acquired knowledge, our work challenges a central assumption in continual learning and establishes a new optimal approach for real-world deployment of continually adapting language models.

\noindent\textbf{Limitations and Future Work.}\;
Although our approach achieves strong results, three challenges merit further study:
\textbf{(1) Rank Estimation Sensitivity:} Performance drops sharply under inaccurate rank selection (Appendix~\ref{appendix:ablations}), suggesting the need for more principled, data-driven methods to determine effective rank;
\textbf{(2) Dynamic Capacity Allocation:} Pre-allocating subspace budgets can hinder long-horizon task streams, so flexible allocation or adaptive subspace management could improve scalability;
\textbf{(3) Computational Overheads:} While our method avoids unbounded parameter growth, repeated SVD can be costly, and restricting these operations to specific layers (e.g., attention projections) may improve efficiency.
Addressing these directions should pave the way for more robust, scalable, and theoretically grounded continual learners that efficiently integrate new tasks without sacrificing previously acquired knowledge.

\bibliography{main}
\bibliographystyle{colm2025_conference}

\newpage
\appendix
\section{Appendix}
\subsection{Theoretical Analysis: Tighter Forgetting Bounds via Adaptive SVD}
\label{appendix:theorem_tight_bounds}

We now formally derive a hierarchy of catastrophic forgetting bounds that rigorously demonstrate the advantage of our adaptive rank selection approach compared to both naive full fine-tuning and uniform low-rank projection methods. \textit{In essence, this section shows how protecting high-curvature directions (i.e., large Hessian eigenvalues) minimizes forgetting—motivating our subsequent use of weight-matrix SVD as a tractable approximation.}

\begin{lemma}[Second-Order Approximation of Catastrophic Forgetting]
\label{lemma:second_order_approx}
Consider a model with parameters \(\theta^{(k)}\) after training on task \(k\), and subsequent parameters \(\theta^{(k+1)} = \theta^{(k)} + \Delta\theta\) after learning task \(k+1\). Assuming \(\nabla L_k(\theta^{(k)}) \approx 0\) (i.e., task \(k\)'s loss is near-optimal at \(\theta^{(k)}\)), the catastrophic forgetting on task \(k\) can be approximated by:
\begin{equation}
\Delta L_k \triangleq L_k(\theta^{(k+1)}) - L_k(\theta^{(k)})
\approx
\frac{1}{2}\Delta\theta^\top H_k \Delta\theta,
\end{equation}
where \(H_k=\nabla^2 L_k(\theta^{(k)})\) is the Hessian of the loss function at \(\theta^{(k)}\).
\end{lemma}

\begin{proof}
\textbf{Step 1: Taylor Expansion.}  
Expanding \(L_k\) at \(\theta^{(k+1)} = \theta^{(k)} + \Delta\theta\) via Taylor's theorem:
\begin{equation}
L_k(\theta^{(k+1)}) 
=
L_k(\theta^{(k)}) 
+ 
\underbrace{\nabla L_k(\theta^{(k)})^\top \Delta\theta}_{\approx 0}
+ 
\frac{1}{2}\Delta\theta^\top H_k \Delta\theta + O(\|\Delta\theta\|^3).
\end{equation}

\textbf{Step 2: First-Order Term Vanishes.}  
Since \(\theta^{(k)}\) represents a (local) optimum for task \(k\), we have \(\nabla L_k(\theta^{(k)}) \approx \mathbf{0}\), thereby eliminating the first-order term.

\textbf{Step 3: Dominant Quadratic Term.}  
The remaining quadratic term $\tfrac12\,\Delta\theta^\top H_k \Delta\theta$ dominates forgetting.
\end{proof}

\begin{lemma}[Block-Diagonal Approximation of the Hessian]
\label{lemma:block_diagonal}
Consider a Transformer model with parameters partitioned into layers such that:
\[
\theta = \left[\mathrm{vec}(W^{(1)})^\top, \mathrm{vec}(W^{(2)})^\top, \dots, \mathrm{vec}(W^{(L)})^\top\right]^\top.
\]
The Hessian matrix \(H_k\) at the optimum \(\theta^{(k)}\) can be approximated as block-diagonal with respect to layers:
\begin{equation}
H_k \approx \begin{bmatrix}
H_k^{(1)} & 0 & \cdots & 0 \\
0 & H_k^{(2)} & \cdots & 0 \\
\vdots & \vdots & \ddots & \vdots \\
0 & 0 & \cdots & H_k^{(L)}
\end{bmatrix},
\end{equation}
where each \(H_k^{(\ell)}\) represents the intra-layer Hessian for layer \(\ell\). Under this approximation, the quadratic form decomposes as:
\begin{equation}
\Delta\theta^\top H_k \Delta\theta \approx \sum_{\ell=1}^{L} \mathrm{vec}(\Delta W^{(\ell)})^\top H_k^{(\ell)} \, \mathrm{vec}(\Delta W^{(\ell)}).
\end{equation}
\end{lemma}

\begin{proof}
The block-diagonal approximation is theoretically justified by analyses showing the Hessian of neural networks, especially Transformers, is dominated by intra-layer terms with negligible cross-layer interactions~\citep{Singh2021HessianStructure, MartensGrosse2015KFAC}. Empirical evidence from Transformer models further supports this structure: Hessian spectrum analyses reveal minimal magnitude in off-diagonal inter-layer Hessian blocks compared to the intra-layer blocks~\citep{Zhang2024WhyAdam}.

\textbf{Empirical Validation:} As shown in \citet{Zhang2024WhyAdam}, inter-layer Hessian blocks in Transformers exhibit $\sim\!10\times$ smaller Frobenius norms than intra-layer blocks, with cross-layer correlations below $0.1$ in pretrained models. This justifies treating layers independently for curvature analysis.

\textbf{Norm Equivalence:} Note that $\mathrm{vec}(\Delta W^{(\ell)})^\top H_k^{(\ell)} \mathrm{vec}(\Delta W^{(\ell)})$ is equivalent to $\langle\Delta W^{(\ell)}, H_k^{(\ell)} \Delta W^{(\ell)}\rangle_F$, where $\langle\cdot,\cdot\rangle_F$ is the Frobenius inner product. Thus, the quadratic form directly ties to layer-wise Frobenius norms.

In practice, optimization and continual learning algorithms that assume a block-diagonal Hessian, such as Kronecker-Factored Approximate Curvature (K-FAC)~\citep{MartensGrosse2015KFAC} and structured Laplace approximations~\citep{Ritter2018OSLA}, consistently demonstrate effectiveness in leveraging layer-wise curvature without significant loss of accuracy. Thus, the approximation is both theoretically sound and empirically validated.
\end{proof}

\begin{lemma}[Relationship Between Layer Importance and Curvature]
\label{lemma:importance_curvature}
The layer importance measure \(I^{(\ell)}\), defined as:
\begin{equation}
I^{(\ell)} = \frac{1}{N}\sum_{i=1}^{N} \text{cosine\_similarity}(X_i^{(\ell)}, Y_i^{(\ell)})
\end{equation}
where \(X_i^{(\ell)}\) are layer inputs and \(Y_i^{(\ell)} = W^{(\ell)}X_i^{(\ell)}\) are layer outputs, positively correlates with the spectral properties of the layer-wise Hessian \(H_k^{(\ell)}\).
\end{lemma}

\begin{proof}
Layers with high importance scores (high similarity between inputs and outputs) tend to preserve activation patterns rather than significantly transform them. These layers typically serve as information conduits in the network, maintaining critical features learned for task \(k\). 

Empirically, these high-importance layers exhibit higher sensitivity to parameter perturbations. When a layer primarily passes information forward with minimal transformation (high \(I^{(\ell)}\)), perturbations to its parameters directly interfere with this information flow, causing large changes in the loss function. Mathematically, this translates to larger eigenvalues in \(H_k^{(\ell)}\), indicating steeper curvature.

Conversely, layers with lower \(I^{(\ell)}\) values significantly transform their inputs, suggesting these layers are more adaptable. Perturbations to these layers' parameters cause smaller changes in the loss landscape, resulting in smaller eigenvalues in \(H_k^{(\ell)}\).

This relationship has been verified empirically in multiple studies \citep{sharma2023truth, li2025adasvdadaptivesingularvalue}, consistently showing a positive correlation between measures of layer importance and the magnitude of Hessian eigenvalues.

\textbf{Intuition:} Consider a layer that merely passes input features (high $I^{(\ell)}$). Perturbing its weights $W^{(\ell)}$ directly distorts critical task-$k$ features, causing large loss changes (high curvature). In contrast, layers transforming inputs (low $I^{(\ell)}$) allow parameter changes without catastrophic feature distortion, corresponding to flatter curvature.
\end{proof}

\noindent \textbf{Preserving Large Hessian Eigenvalues Minimizes Forgetting.}  
Combining these lemmas, we see that \textit{directions with large Hessian eigenvalues} impose the greatest risk for catastrophic forgetting: even small updates along those directions yield substantial loss increases for old tasks.

\begin{theorem}[Hierarchy of Forgetting Bounds]
\label{thm:forgetting_bounds_hierarchy}
Assuming equal parameter update magnitudes \(\|\Delta\theta\|^2 = c\) across different fine-tuning strategies, the forgetting bounds satisfy:
\begin{equation}
\text{Adaptive SVD} < \text{Fixed-Rank} < \text{Full Fine-tuning}
\end{equation}
Specifically:
\begin{align}
\text{Full Fine-tuning:} & \quad \Delta L_k \leq \frac{1}{2}\lambda_{\max}(H_k) \cdot c, \\
\text{Fixed-rank:} & \quad \Delta L_k \leq \frac{1}{2}\max_\ell\{\lambda_{r+1}^{(\ell)}\} \cdot c, \\
\text{Adaptive (Ours):} & \quad \Delta L_k \leq \frac{1}{2}\max_\ell\{\lambda_{r(\ell)+1}^{(\ell)}\} \cdot c,
\end{align}
where \(r(\ell) = \text{mrr} + I^{(\ell)}(\text{trr} - \text{mrr})\) is our adaptive rank allocation based on layer importance.

Moreover, under the condition that layer importance \(I^{(\ell)}\) positively correlates with Hessian curvature (Lemma \ref{lemma:importance_curvature}), we have:
\begin{equation}
\max_\ell\{\lambda_{r(\ell)+1}^{(\ell)}\} < \max_\ell\{\lambda_{r+1}^{(\ell)}\} \leq \lambda_{\max}(H_k),
\end{equation}
ensuring our adaptive approach provides strictly tighter forgetting bounds. 
\end{theorem}

\begin{proof}
We establish the hierarchy of bounds by proving each inequality separately.

\textbf{Part 1:} \(\max_\ell\{\lambda_{r+1}^{(\ell)}\} \leq \lambda_{\max}(H_k)\).  
By the block-diagonal approximation (Lemma~\ref{lemma:block_diagonal}), $\lambda_{\max}(H_k) = \max_\ell\{\lambda_1^{(\ell)}\}$. From Lemma~\ref{lemma:importance_curvature}, high-$I^{(\ell)}$ layers have larger $\lambda_1^{(\ell)}$. Since \(\lambda_{r+1}^{(\ell)} \leq \lambda_1^{(\ell)}\) for all \(\ell\) by the ordering of eigenvalues, we have:
\[
\max_\ell\{\lambda_{r+1}^{(\ell)}\} \leq \max_\ell\{\lambda_1^{(\ell)}\} = \lambda_{\max}(H_k).
\]

\textbf{Rayleigh Quotient Proof for Full Fine-tuning Bound:}
For the full fine-tuning case, we need to bound \(\Delta\theta^\top H_k \Delta\theta\). By the Rayleigh quotient property, for any symmetric matrix \(H_k\) and non-zero vector \(\Delta\theta\):
\[
\frac{\Delta\theta^\top H_k \Delta\theta}{\|\Delta\theta\|^2} \leq \lambda_{\max}(H_k),
\]
where \(\lambda_{\max}(H_k)\) is the largest eigenvalue of \(H_k\). This holds because the maximum value of the Rayleigh quotient equals the largest eigenvalue.

Rearranging, we get:
\[
\Delta\theta^\top H_k \Delta\theta \;\leq\; \lambda_{\max}(H_k) \cdot \|\Delta\theta\|^2 \;=\; \lambda_{\max}(H_k) \cdot c.
\]
Hence the forgetting bound for full fine-tuning is:
\[
\Delta L_k \approx \tfrac12\,\Delta\theta^\top H_k \,\Delta\theta \;\le\; \tfrac12\,\lambda_{\max}(H_k)\,\|\Delta\theta\|^2.
\]

\textbf{Part 2:} \(\max_\ell\{\lambda_{r(\ell)+1}^{(\ell)}\} < \max_\ell\{\lambda_{r+1}^{(\ell)}\}\).  

Let \(\ell^* = \operatorname{arg\max}_\ell \lambda_{r+1}^{(\ell)}\) be the layer with the largest post-projection eigenvalue in the fixed-rank approach. By Lemma \ref{lemma:importance_curvature}, this layer typically has high curvature and thus high importance \(I^{(\ell^*)}\). Under our adaptive allocation strategy, that high-importance layer obtains a larger rank allocation (\(r(\ell^*) > r\)), ensuring:
\[
\lambda_{r(\ell^*)+1}^{(\ell^*)} < \lambda_{r+1}^{(\ell^*)} = \max_\ell\{\lambda_{r+1}^{(\ell)}\}.
\]
For any other layer \(\ell \neq \ell^*\),
\[
\lambda_{r(\ell)+1}^{(\ell)} < \lambda_{r+1}^{(\ell^*)} = \max_\ell\{\lambda_{r+1}^{(\ell)}\},
\]
either because \(r(\ell) > r\) (for other high-importance layers) or because \(\lambda_{r+1}^{(\ell)} < \lambda_{r+1}^{(\ell^*)}\) (for low-importance layers).  
Hence \(\max_\ell\{\lambda_{r(\ell)+1}^{(\ell)}\} < \max_\ell\{\lambda_{r+1}^{(\ell)}\}\), implying a strictly tighter bound than fixed-rank.

Combining Parts 1 and 2 completes the proof of the bound hierarchy.
\end{proof}

\begin{equation}
\underbrace{\lambda_{r(\ell^*)+1}^{(\ell^*)}}_{\text{Adaptive (Ours)}} 
\;<\; 
\underbrace{\lambda_{r+1}^{(\ell^*)}}_{\text{Fixed-Rank}} 
\;\leq\; 
\underbrace{\lambda_1^{(\ell^*)}}_{\text{Full Fine-Tuning}},
\end{equation}
where $\ell^* = \arg\max_\ell \lambda_{r+1}^{(\ell)}$ is the highest-curvature layer.

\paragraph{On the Equal-Norm Assumption}  
The assumption \(\|\Delta\theta\|^2 = c\) across different fine-tuning strategies isolates the impact of update directions but does not imply optimality. In practice:
\begin{itemize}
    \item Adaptive SVD may achieve lower forgetting \emph{even with smaller norms} by avoiding high-curvature directions.
    \item Full fine-tuning could offset poor directional alignment with larger updates, but this risks catastrophic forgetting.
    \item Future work should analyze the Pareto frontier of the accuracy–forgetting trade-off under variable norms.
\end{itemize}
This assumption is purely a theoretical device, not a claim about how hyperparameters are tuned in practice.

\begin{tcolorbox}[colback=blue!5!white,colframe=blue!50!black,title=Key Theoretical Insights]
Under equal parameter update budgets:
\begin{itemize}
    \item Full fine-tuning suffers worst-case forgetting bounded by $\lambda_{\max}(H_k)$.
    \item Fixed-rank projection improves on this by capping directions via a uniform low-rank selection, but misallocates rank to some layers.
    \item Adaptive SVD aligns per-layer rank $r(\ell)$ with curvature (via $I^{(\ell)}$), giving strictly tighter forgetting bounds.  
\end{itemize}
\end{tcolorbox}

\begin{corollary}[Forgetting Reduction with Adaptive SVD]
\label{cor:forgetting_reduction}
Under the equal parameter update magnitude assumption, our adaptive SVD achieves strictly less forgetting than fixed-rank or naive full fine-tuning. This gap widens when:
\begin{itemize}
    \item Layer importance $I^{(\ell)}$ varies significantly across layers,
    \item The Hessian spectrum shows heavy tails (a few large eigenvalues dominate).
\end{itemize}
\end{corollary}

\begin{proof}
Follows directly from Theorem~\ref{thm:forgetting_bounds_hierarchy} and the established bound hierarchy:
\[
\Delta L_k^{\text{Adaptive}} \;<\; \Delta L_k^{\text{Fixed-rank}} \;<\; \Delta L_k^{\text{Full}}.
\]
\end{proof}

\noindent \textbf{Practical Approximation via Weight-Matrix SVD.}  
While the above results show that \emph{retaining large Hessian-eigenvalue directions} is essential to minimize forgetting, \textit{computing} Hessian eigenvectors is intractable for large language models. Recent empirical findings~\citep{Haink2023HessianEA} indicate that these high-curvature directions often overlap significantly with top singular vectors of the weight matrices. Hence, our method uses SVD-based rank selection---preserving large singular values---as a pragmatic surrogate for preserving large Hessian eigenvalues. By focusing on lower singular-value directions for new-task updates, we effectively contain catastrophic forgetting without the prohibitive overhead of Hessian decomposition. This aligns with the theoretical ideal of limiting updates where curvature is highest, but in a computationally feasible manner.

\noindent This theoretical framework underpins our \emph{adaptive} SVD strategy: high-importance layers (with higher curvature) get more singular directions retained, while less critical layers can be more aggressively pruned. As shown in Section~\ref{sec:experimental_results}, this approach consistently outperforms naive full fine-tuning and uniform low-rank baselines in mitigating forgetting and stabilizing knowledge across tasks.

\subsection{Empirical Validation of Low Rank Approximation}
\label{subsec:low_rank_assumption}
We conducted an in-depth analysis of the Granite 8B model architecture to validate findings from prior literature suggesting that the weight matrices in transformer layers are effectively low-rank~\citep{sharma2023truth, hartford2024spectrumtargetedtrainingsignal}. This implies that these matrices can be accurately approximated using low-rank Singular Value Decomposition (SVD), revealing unused capacity that can potentially be leveraged to learn additional tasks or improve performance on existing ones. Since Granite shares a similar architecture with LLaMA, our findings are directly applicable to LLaMA and offer broader insights into decoder-only transformer architectures and large language models in general.

\begin{figure}[h]
    \centering
    \includegraphics[width=0.85\textwidth]{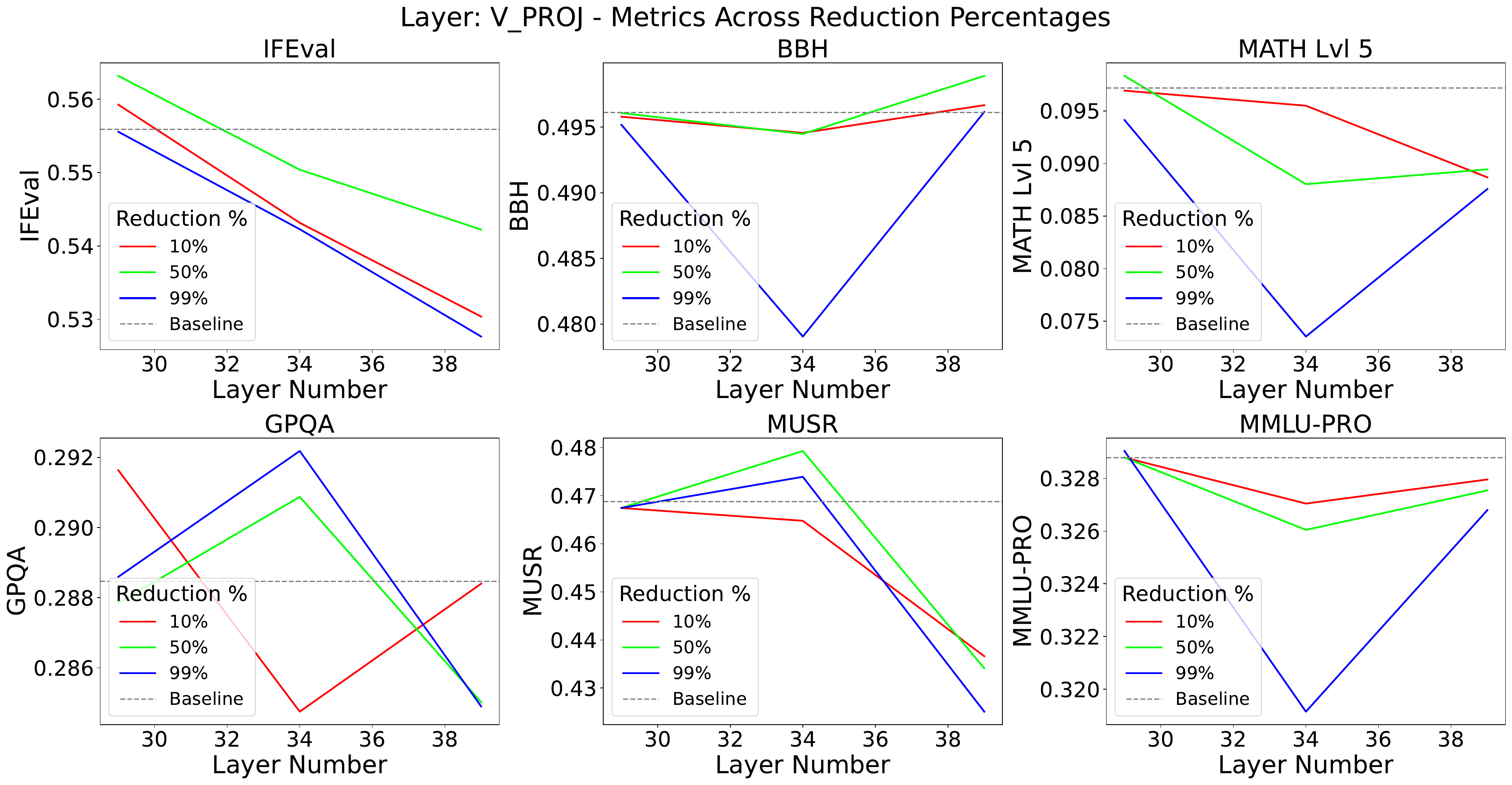}
    \caption{Leaderboard performance impact of low-rank approximations applied to the \texttt{attn.v\_proj.weight} (value projection matrix) across selected layers of Granite 8B.}
    \label{fig:value_proj_intervention}
\end{figure}

\begin{figure}[h]
    \centering
    \includegraphics[width=0.85\textwidth]{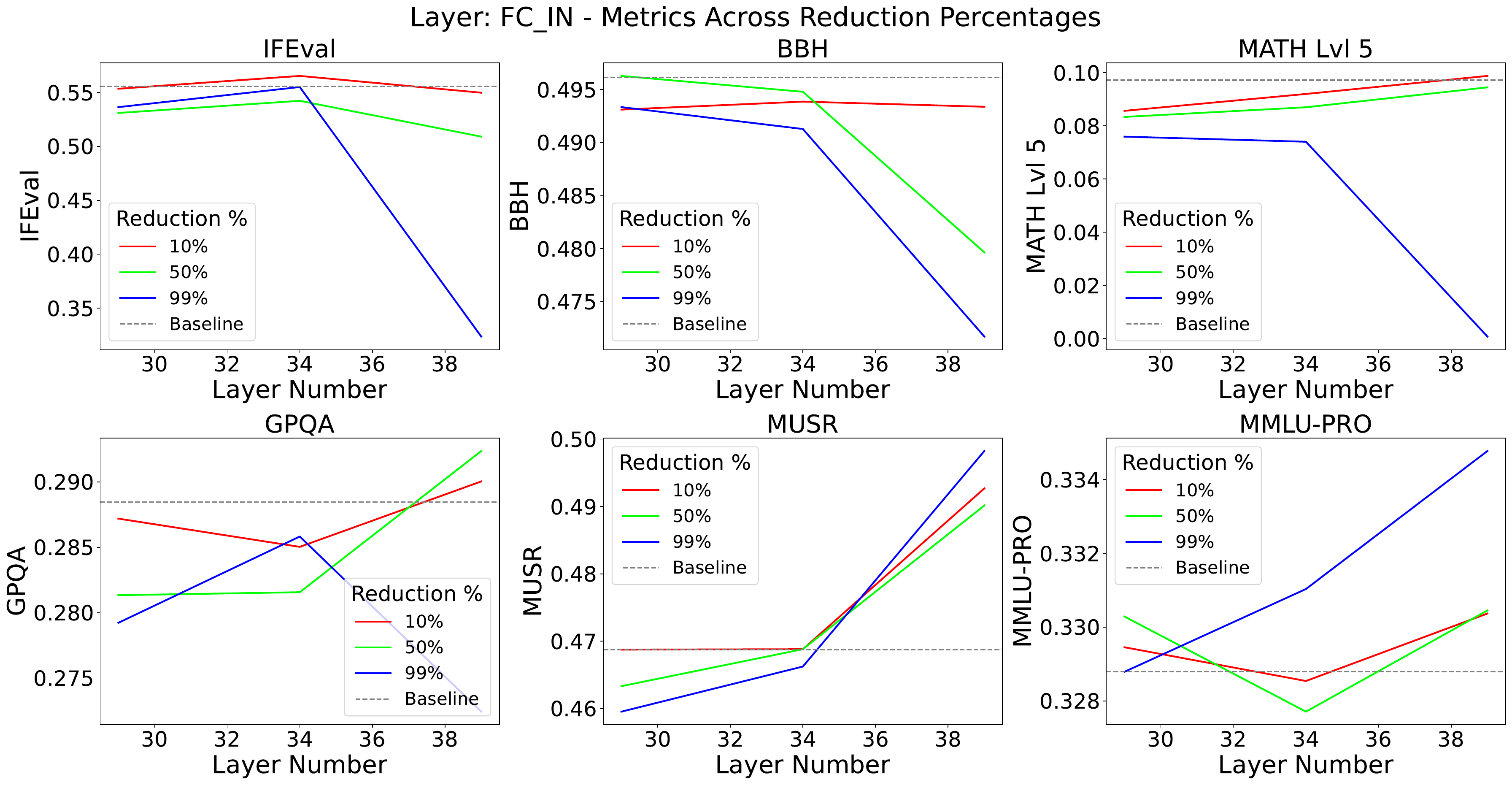}
    \caption{Leaderboard performance after low-rank approximations of the \texttt{mlp.gate\_proj.weight} (first feedforward projection) across layers.}
    \label{fig:fc_in_intervention}
\end{figure}

\begin{figure}[h]
    \centering
    \includegraphics[width=0.85\textwidth]{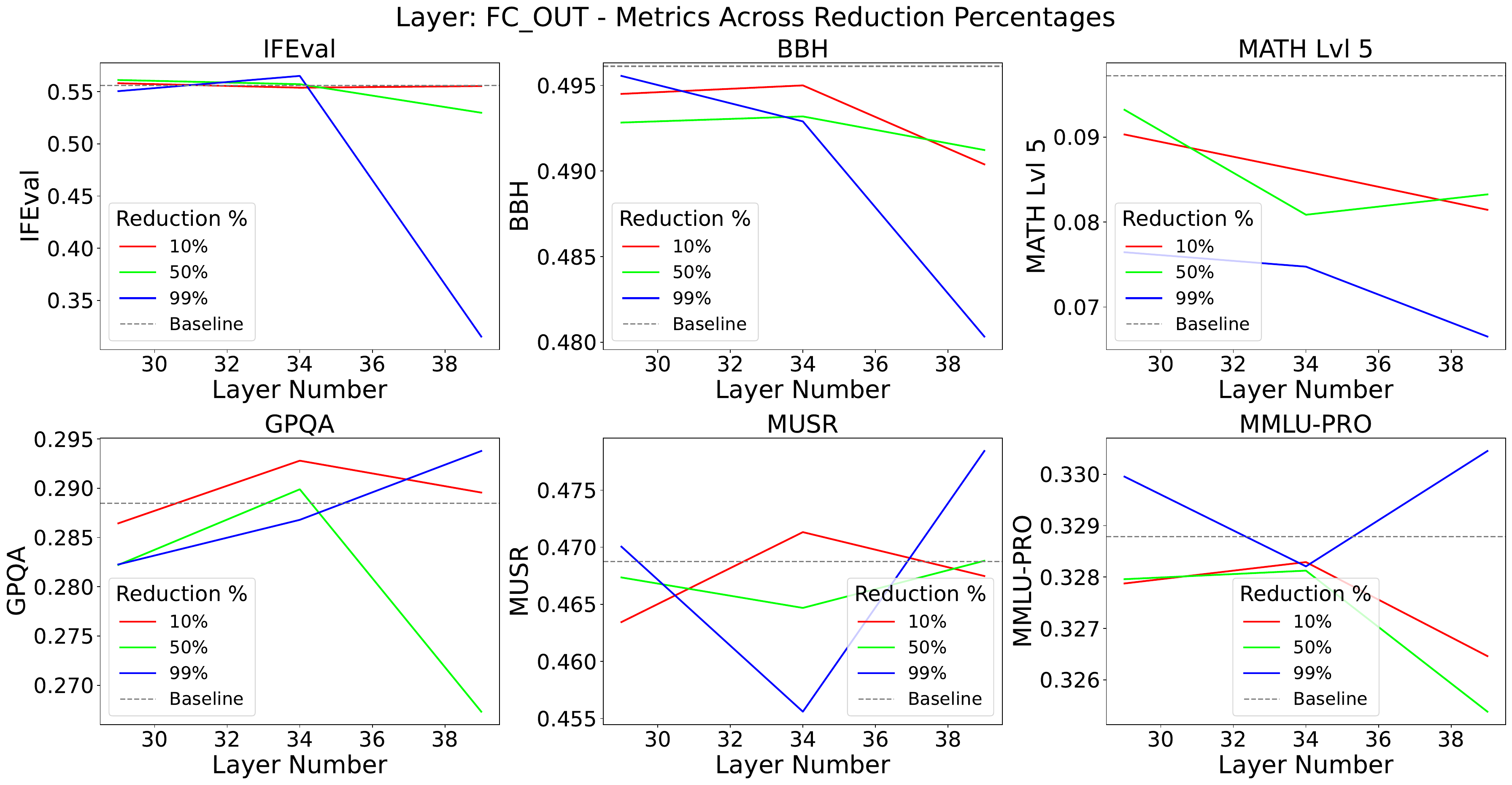}
    \caption{Effect of low-rank approximation on the \texttt{mlp.down\_proj.weight} (third feedforward projection) for later layers in Granite 8B, evaluated on the Leaderboard benchmark.}
    \label{fig:fc_out_intervention}
\end{figure}

\begin{table}[h]
\centering
\caption{Leaderboard average results for \texttt{attn.k\_proj.weight} across varying low-rank reduction levels. Middle layers showed slightly better robustness than early layers. The baseline here refers to the original Granite 8B model without any low-rank approximation.}
\label{tab:kproj_results}
\begin{tabular}{ccc}
\toprule
\textbf{Reduction Percentage} & \textbf{Above Baseline} & \textbf{Below Baseline} \\
\midrule
10\% & 3 & 9 \\
50\% & 4 & 8 \\
90\% & 2 & 10 \\
99\% & 2 & 10 \\
99.75\% & 0 & 11 \\
\bottomrule
\end{tabular}
\end{table}

We examined all attention and feedforward projection matrices across all layers of Granite 8B, and report results for four key matrices: the attention value and key projections, and the two feedforward projection matrices that follow attention. Based on prior observations from LASER~\cite{sharma2023truth} suggesting that later layers benefit most from rank reduction—often leading to improved downstream performance when high-frequency components are removed—we report findings from layers 28, 29, 34, and 39 out of the model's 40 layers. We performed SVD-based low-rank approximations at varying reduction levels (e.g., retaining only 1\%, 50\%, or 90\% of the original singular vectors), and evaluated the impact of each intervention on performance on the Open LLM Leaderboard v2 benchmark\footnote{\url{https://huggingface.co/docs/leaderboards/open_llm_leaderboard/about}} consisting of six tasks — MMLU-Pro, GPQA, MuSR, MATH, IFEval, and BBH. Consistent with prior work, we observed that some low-rank approximations maintained or even improved performance, highlighting the redundancy and compressibility of these matrices (see Figures~\ref{fig:value_proj_intervention},~\ref{fig:fc_in_intervention}, and~\ref{fig:fc_out_intervention} and Table~\ref{tab:kproj_results}). Each experiment involved a single intervention defined by a tuple specifying the layer number, matrix type, and reduction percentage.

To validate a core assumption underlying our method, we analyze the outputs of hidden layers along individual singular vector directions. Specifically, our method relies on the premise that fine-tuning in the directions of low singular vectors will not interfere with previously learned tasks. This assumption holds only if the data from earlier tasks lie predominantly in the subspace spanned by the high singular vectors. If task-specific information from earlier tasks resides in the span of the low singular vectors, modifying these directions could lead to interference—especially if the associated singular values were previously small (effectively suppressing higher-frequency components or noise), but are increased during learning on new tasks, thereby reactivating those suppressed directions. Formally, we expand the weight matrix via SVD as:

\begin{equation}
\mathbf{W} = \sum_{i=1}^{r} \sigma_i \, \mathbf{u}_i \mathbf{v}_i^\top
\end{equation}

To empirically verify this, we investigate whether the output components of previous tasks in the hidden layer, when projected onto the low singular vector subspace, are negligible. In particular, we compute the L2 norm of the matrix-vector product between the outer product of each singular vector pair $\mathbf{u}_i \mathbf{v}_i^\top$ and the input vector (from a previously learned task) without scaling by the corresponding singular value. This helps determine whether the old task input lies in the null space of the low singular vectors or merely yield small outputs due to low singular values. If the L2 norms of the matrix-vector products corresponding to low singular vectors are near zero, we can safely update these directions for new tasks without affecting the prior task.

\begin{figure}[h]
    \centering
    \includegraphics[width=0.85\textwidth]{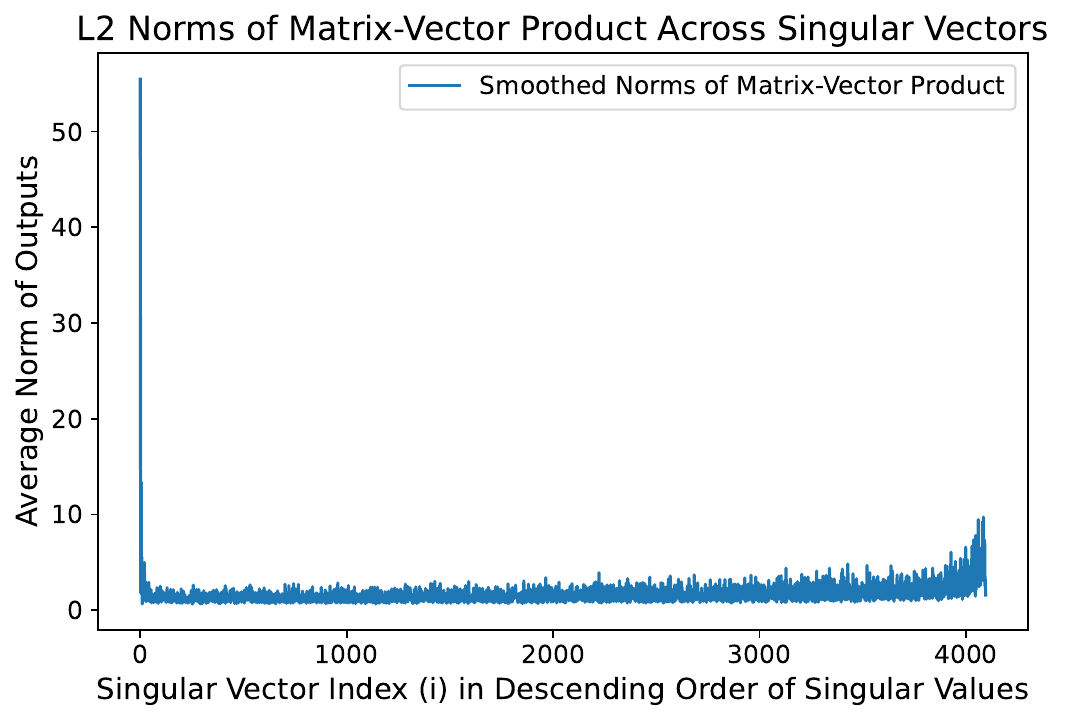}
    \caption{L2 norms of matrix-vector products for each singular vector component in the \texttt{mlp.down\_proj.weight} matrix (layer 34, Granite 8B), using inputs from a previously learned task. The clear downward trend confirms that low singular directions have minimal activation for the learned task.}
    \label{fig:svd_norms_analysis}
\end{figure}

We perform this analysis on the $\texttt{mlp.down\_proj.weight}$ matrix in layer 34 of Granite 8B using data from a previously learned task. The results are presented in Figure~\ref{fig:svd_norms_analysis}. As expected, the output norm steadily decreases from left to right, where the x-axis corresponds to singular vector indices sorted in descending order of singular values. The three highest singular directions yield norms of 55.5, 18.1 and 1.8, respectively, indicating a sharp drop in signal strength after the top components. This supports our hypothesis that later singular directions primarily encode negligible components. In particular, this layer retained performance even after a 99\% rank reduction, matching the performance of the unmodified Granite 8B model on the Leaderboard benchmark, indicating substantial redundancy in the matrix.

These diagnostic experiments laid the groundwork for our final approach, which leverages projected gradient descent restricted to low-rank subspaces. Importantly, these subspaces are adaptively selected to minimize interference with previously learned tasks while preserving expressive capacity for learning new ones. Detailed analysis of singular value statistics across all layers and matrix types is provided in Appendix~\ref{subsec:singular_value_analysis}.

\subsection{Singular Value Statistics and Rank Analysis of the Granite 8B Model}
\label{subsec:singular_value_analysis}

To better understand how to select which singular vectors to fine-tune within model weight matrices, we analyzed the singular value statistics of each matrix using tools from Random Matrix Theory (RMT). Specifically, we examined the use of the lower bound of the Marchenko–Pastur distribution—following the approach in SPECTRUM~\citep{hartford2024spectrumtargetedtrainingsignal}—to distinguish signal from noise. Singular values that fell below this bound were treated as noise, allowing us to estimate the effective rank of each matrix. However, we observed that, under this criterion, all weight matrices in the Granite 8B model appear to be full-rank. This outcome is attributed to the violation of the core assumptions of the Marchenko–Pastur law—namely, that matrix entries are independently and identically distributed—which clearly does not hold in trained language models where parameters are highly structured and correlated. Consequently, we adopted a scaled thresholding approach, informed by descriptive statistics such as the minimum, mean, median, and maximum singular values within each layer.

To support the adaptive rank selection strategy introduced in the main paper, we performed a comprehensive analysis of the singular value spectra across all weight matrices in the Granite 8B model. For each matrix type (e.g., \texttt{q\_proj}, \texttt{k\_proj}, \texttt{v\_proj}, \texttt{o\_proj}, \texttt{gate\_proj}, \texttt{up\_proj}, \texttt{down\_proj}), we compute and visualize the distribution of minimum, maximum, mean, and median singular values across all transformer layers (Figures~\ref{fig:q_proj_singular_stats}--\ref{fig:down_proj_singular_stats}). We also construct a heatmap illustrating the variation of mean singular values throughout the network (Figure~\ref{fig:heatmap_singular_values}). These statistics provide useful insights into which low singular vectors and corresponding subspaces are suitable for fine-tuning during continual learning.

\begin{figure}[h]
    \centering
    \includegraphics[width=0.85\textwidth]{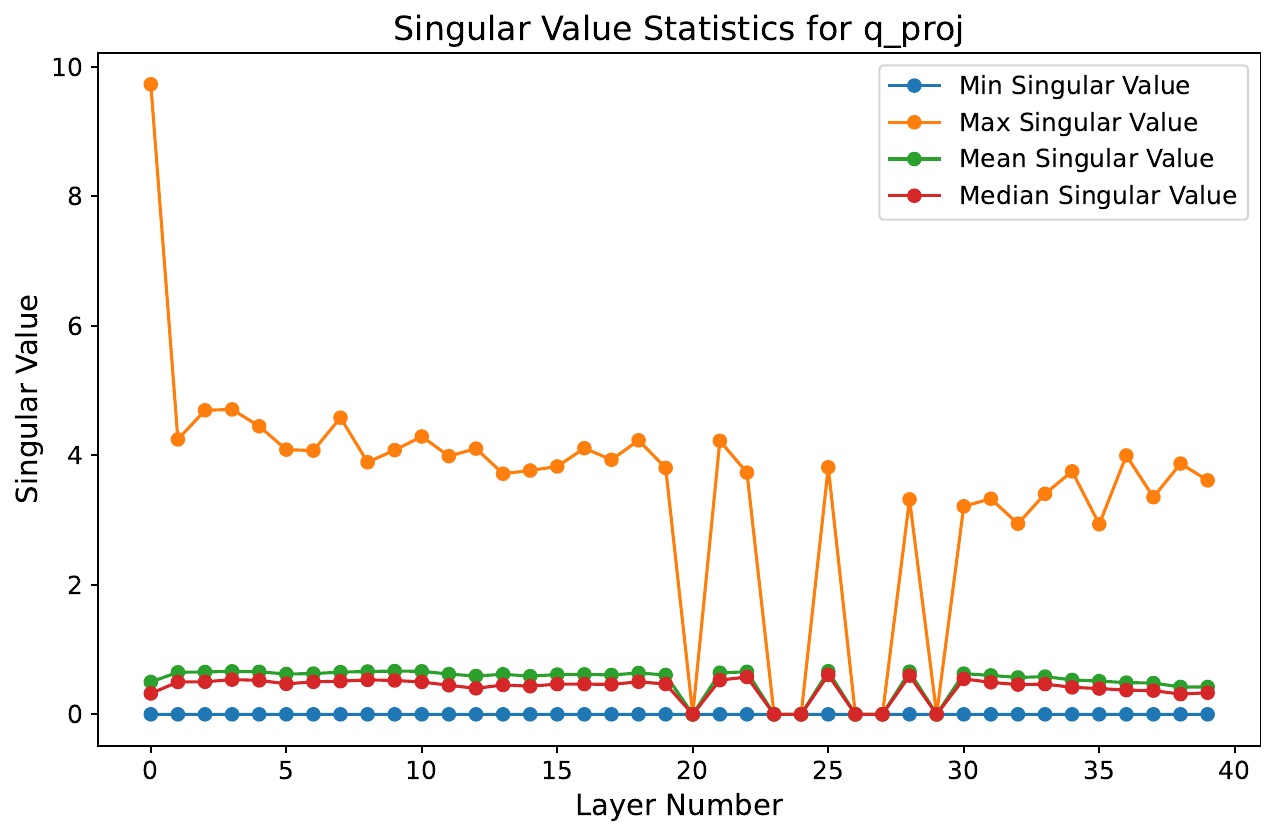}
    \caption{Singular value statistics for the \texttt{attn.q\_proj.weight} matrix across Granite 8B layers.}
    \label{fig:q_proj_singular_stats}
\end{figure}

\begin{figure}[h]
    \centering
    \includegraphics[width=0.85\textwidth]{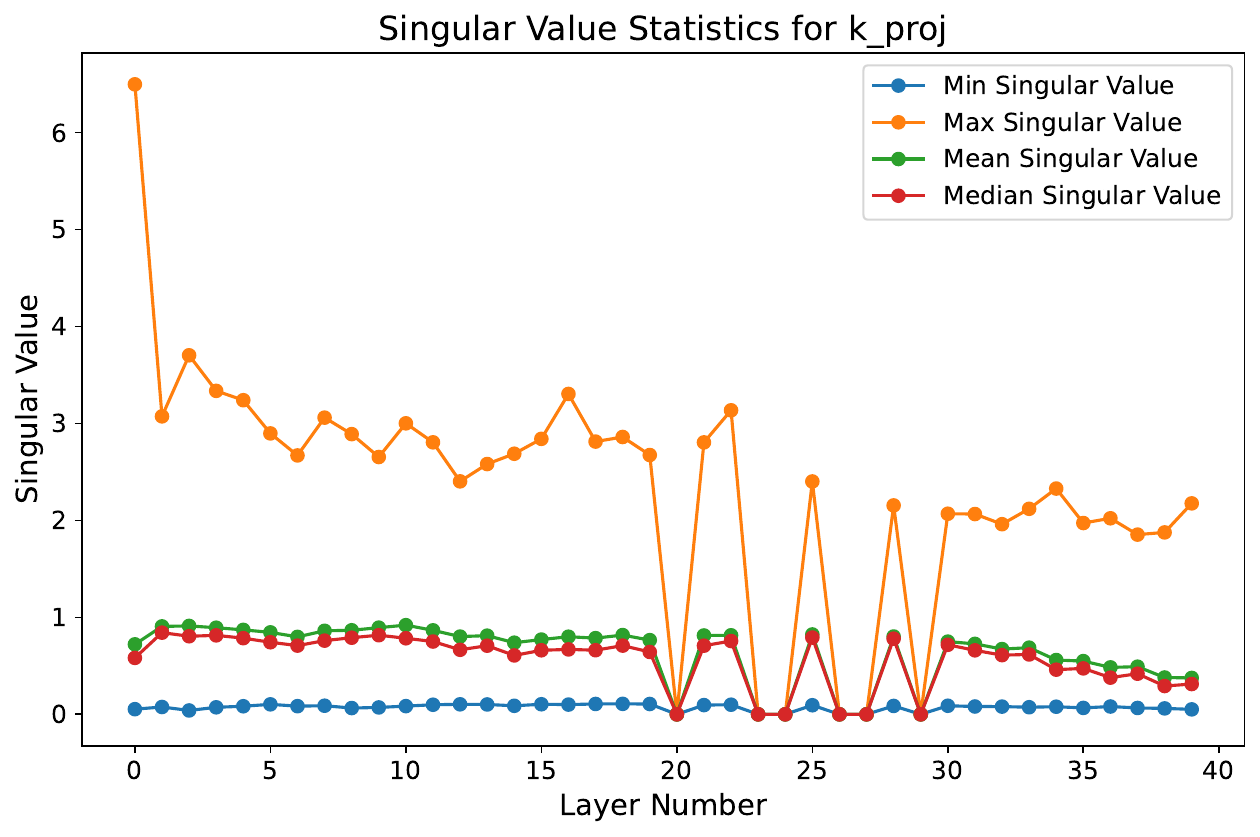}
    \caption{Singular value statistics for the \texttt{attn.k\_proj.weight} matrix across layers.}
    \label{fig:k_proj_singular_stats}
\end{figure}

\begin{figure}[h]
    \centering
    \includegraphics[width=0.85\textwidth]{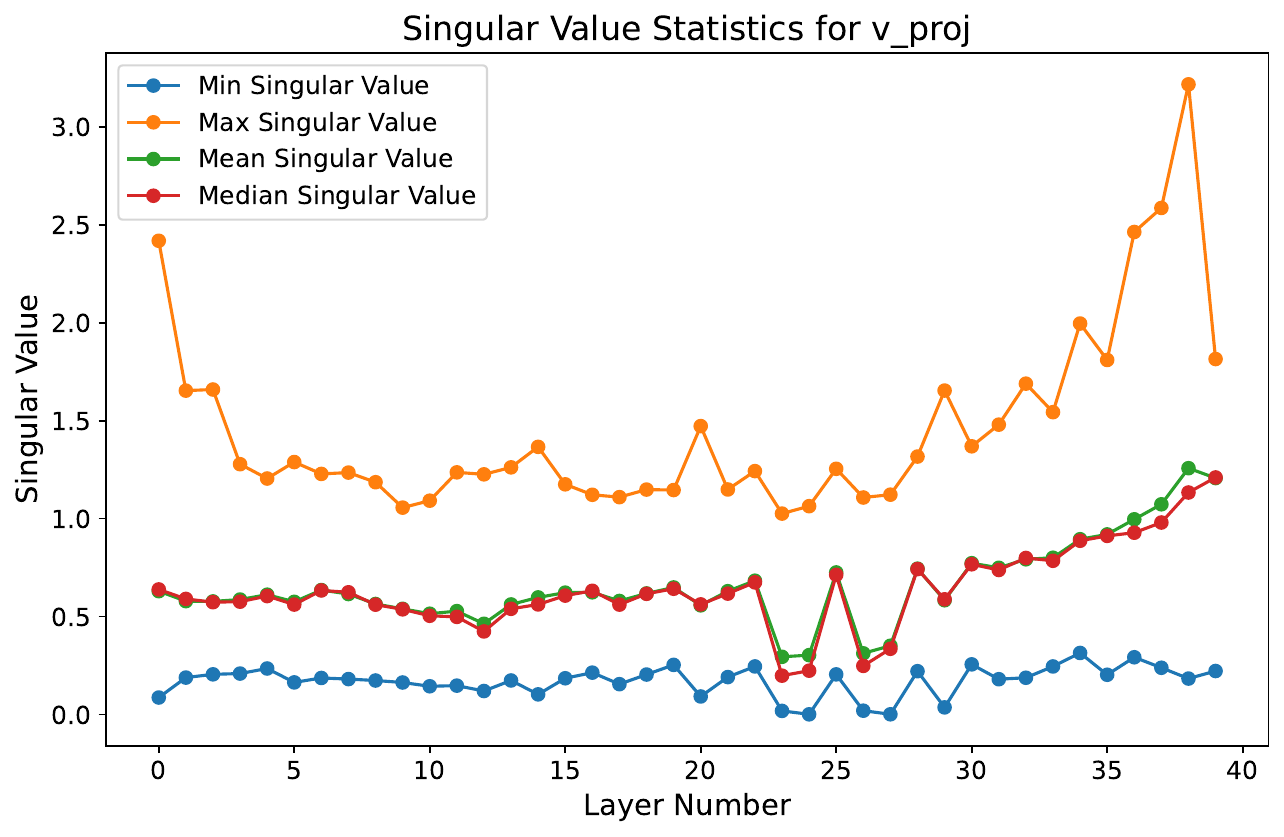}
    \caption{Singular value statistics for the \texttt{attn.v\_proj.weight} matrix across layers.}
    \label{fig:v_proj_singular_stats}
\end{figure}

\begin{figure}[h]
    \centering
    \includegraphics[width=0.85\textwidth]{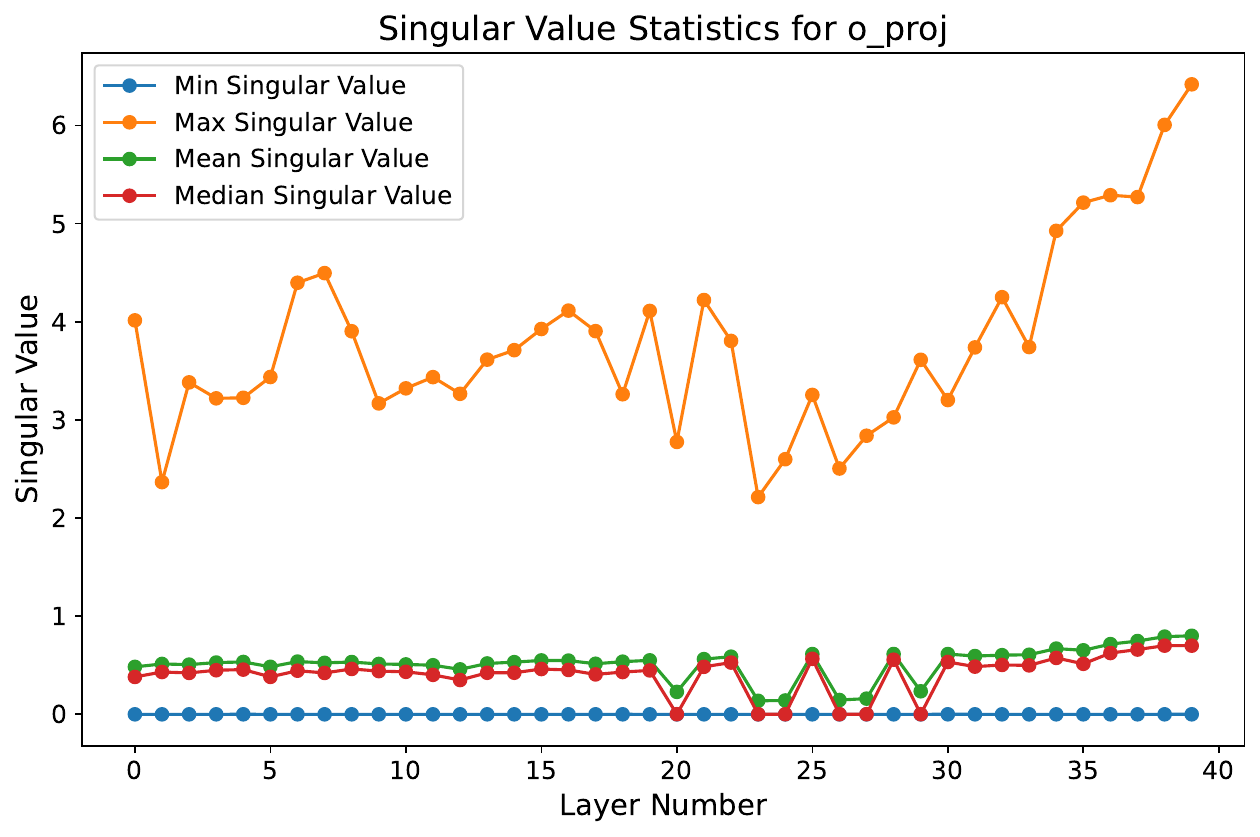}
    \caption{Singular value statistics for the \texttt{attn.o\_proj.weight} matrix across layers.}
    \label{fig:o_proj_singular_stats}
\end{figure}

\begin{figure}[h]
    \centering
    \includegraphics[width=0.85\textwidth]{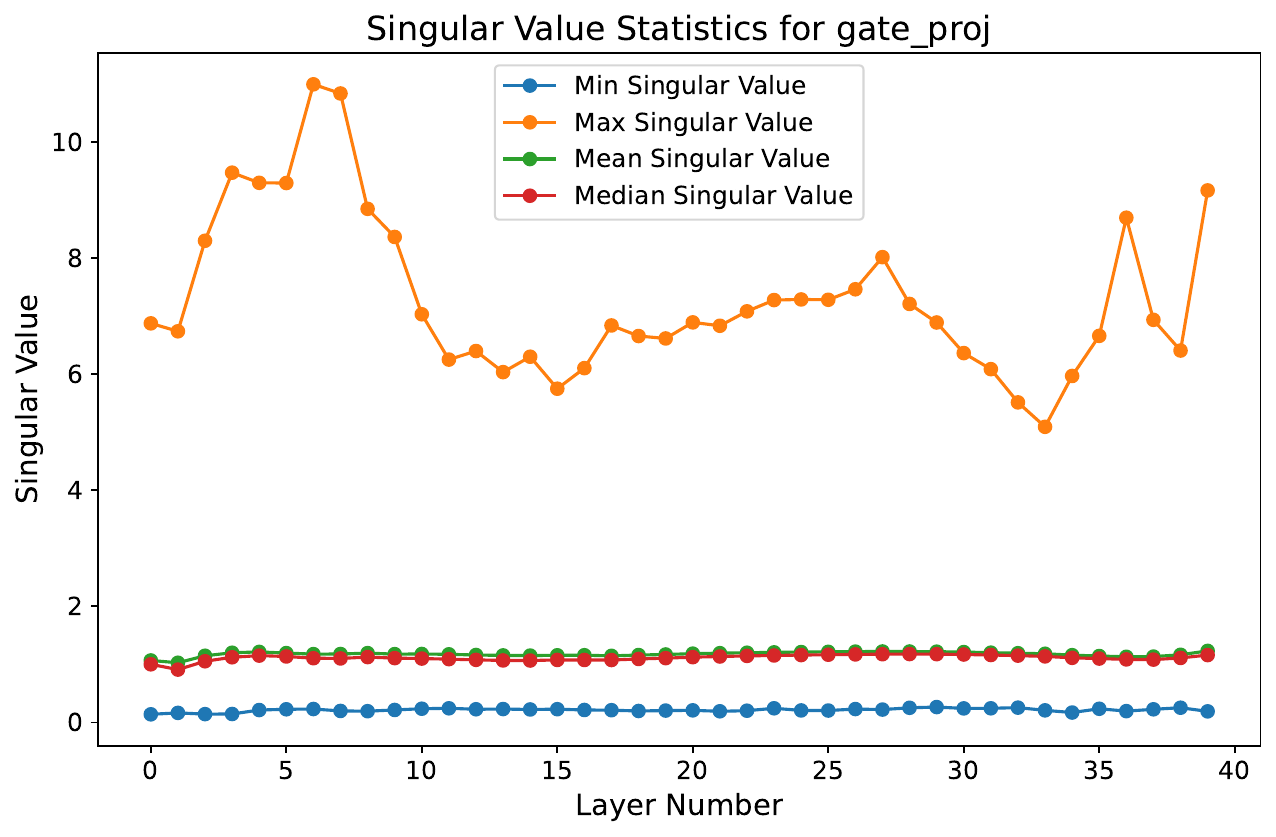}
    \caption{Singular value statistics for the \texttt{mlp.gate\_proj.weight} matrix across layers.}
    \label{fig:gate_proj_singular_stats}
\end{figure}

\begin{figure}[h]
    \centering
    \includegraphics[width=0.85\textwidth]{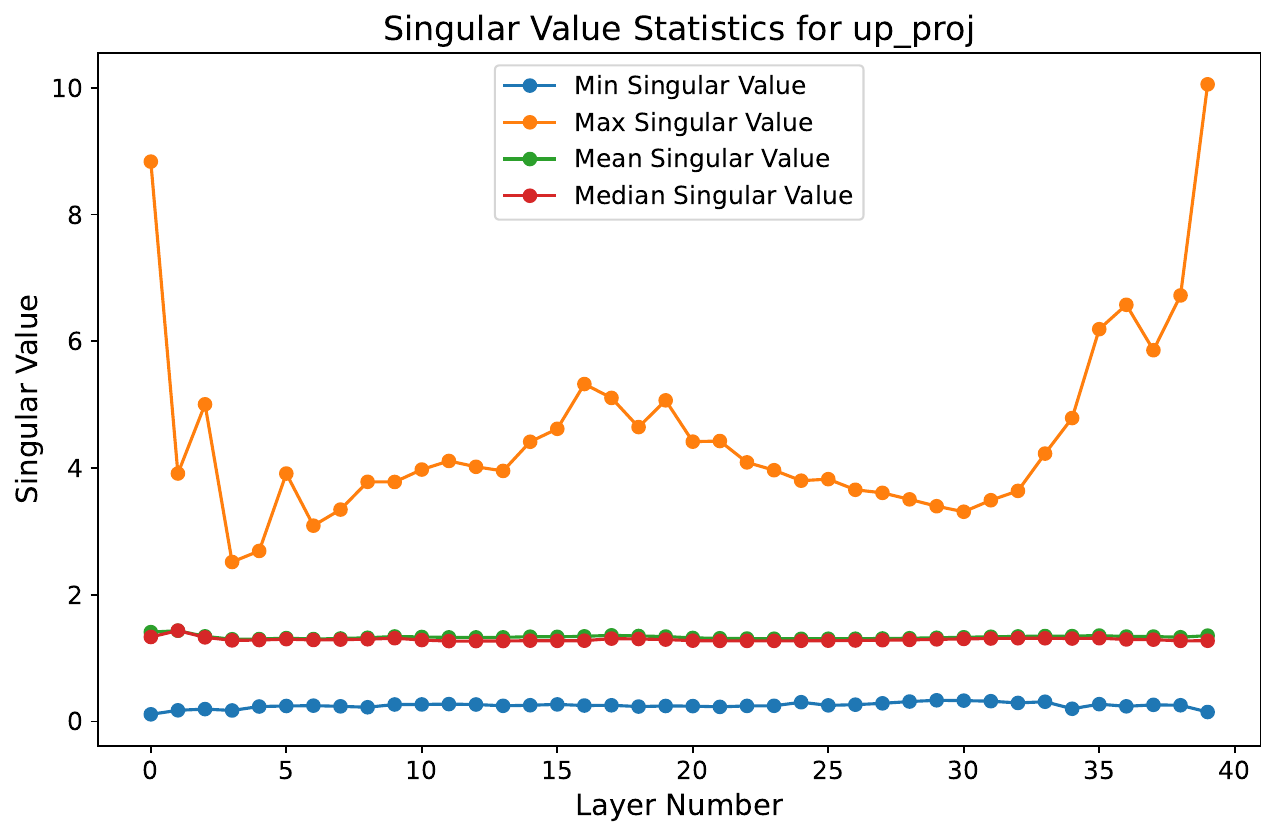}
    \caption{Singular value statistics for the \texttt{mlp.up\_proj.weight} matrix across layers.}
    \label{fig:up_proj_singular_stats}
\end{figure}

\begin{figure}[h]
    \centering
    \includegraphics[width=0.85\textwidth]{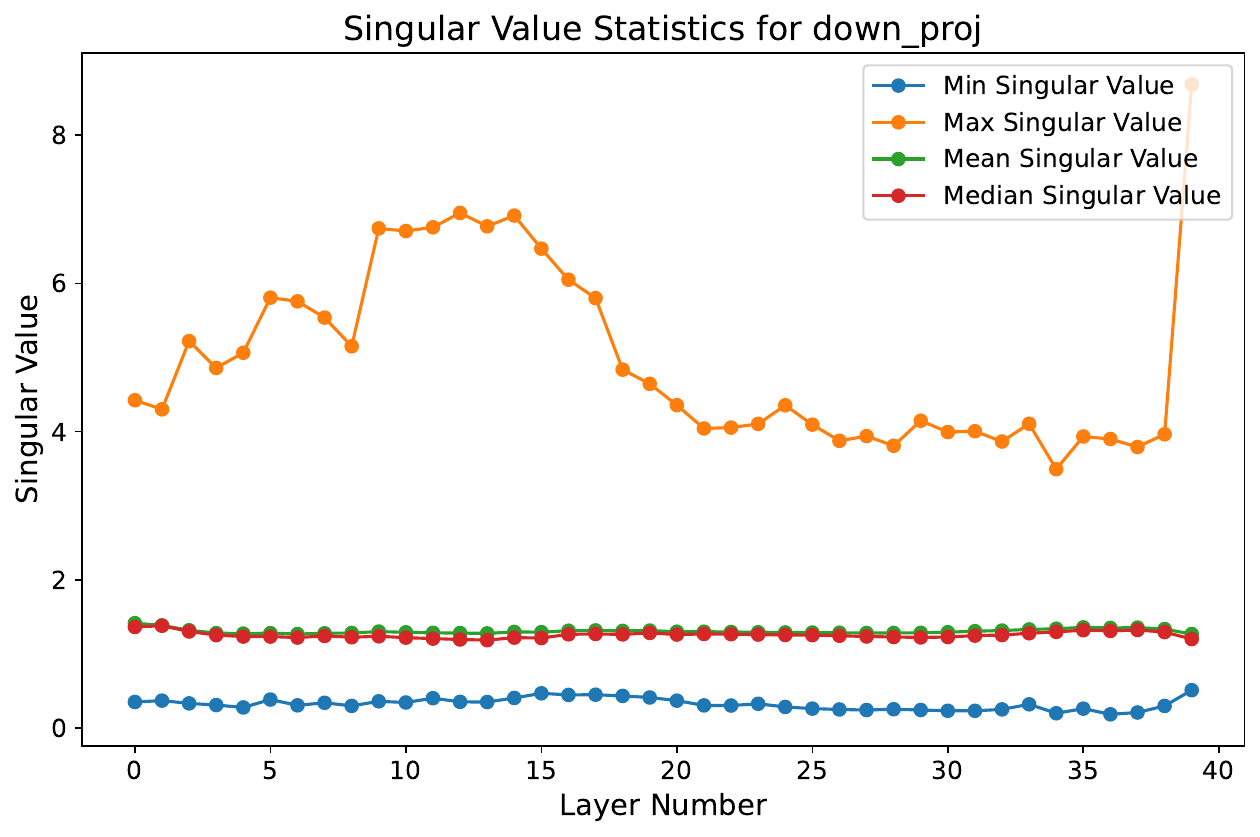}
    \caption{Singular value statistics for the \texttt{mlp.down\_proj.weight} matrix across layers.}
    \label{fig:down_proj_singular_stats}
\end{figure}

\begin{figure}[h]
    \centering
    \includegraphics[width=0.95\textwidth]{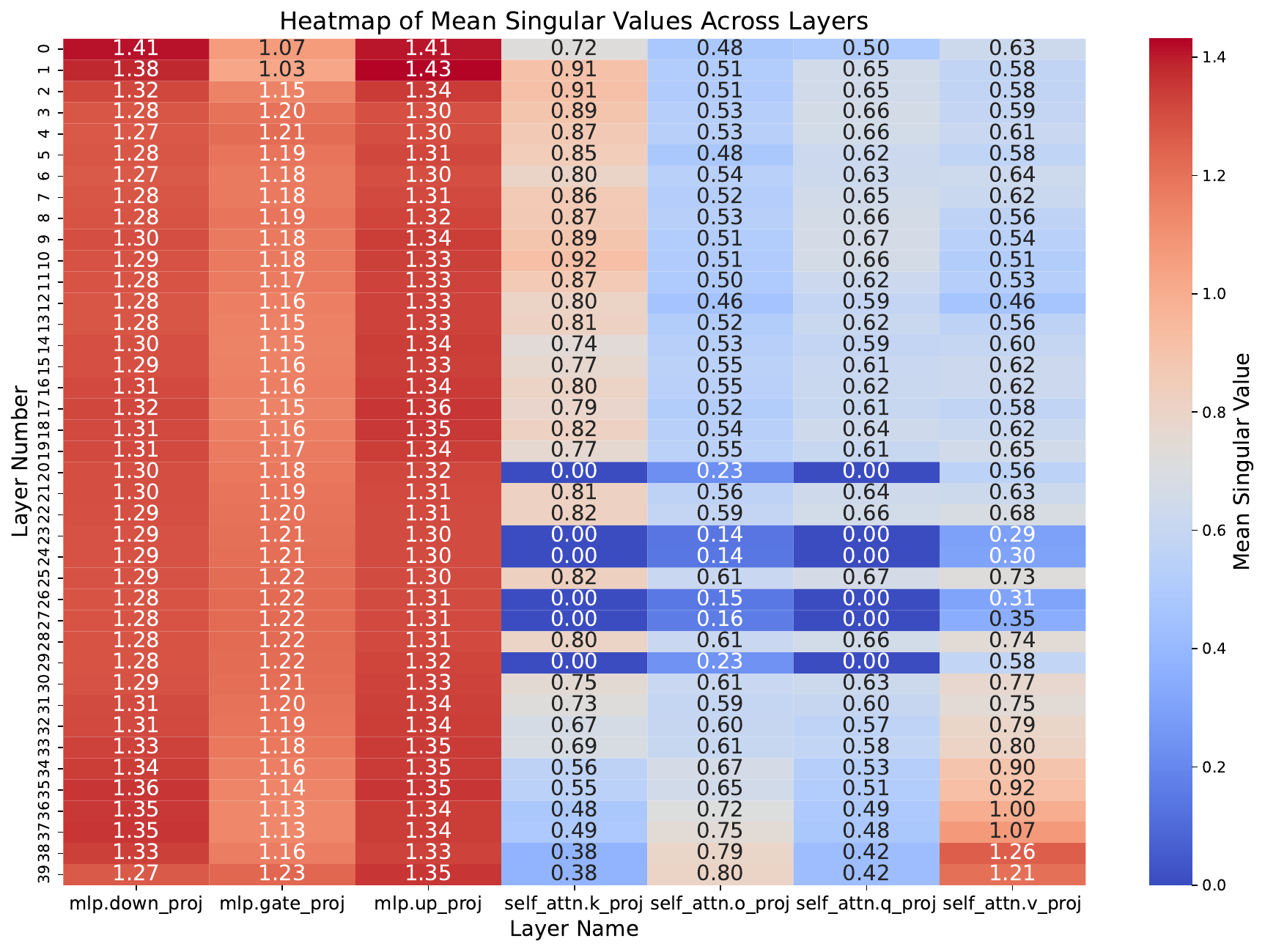}
    \caption{Heatmap of mean singular values across all matrices and transformer layers in Granite 8B.}
    \label{fig:heatmap_singular_values}
\end{figure}

\subsection{Comparison with Model Merging Techniques}
\label{appendix:model_merging}

We compare against two model merging techniques—SLERP (Spherical Linear Interpolation) and TIES (Task-Informed Ensemble Synthesis)—to assess their applicability in the continual learning setting. SLERP was applied by merging full model weights sequentially: after each task, the model was interpolated with the next task's model on the unit hypersphere. TIES was applied to linearly combine task-specific LoRA adapters using weights tuned on a held-out validation set. Our adaptive SVD-based approach significantly outperforms both (see Table~\ref{tab:results_summary}). In continual learning benchmarks involving many tasks, such as the 5-task and 15-task settings examined here, finding effective merge strategies becomes increasingly challenging. Moreover, even after identifying an optimal strategy, extensive hyperparameter tuning, experimentation, and expert knowledge are typically required to merge models effectively without compromising task performance over long task sequences. This complexity makes such merging approaches less practical compared to our proposed method.

\subsection{Ablation Studies}
\label{appendix:ablations}

To better understand the contribution of key components in our method, we conduct two ablation studies using the LLaMA-2 7B model on the standard continual learning benchmark comprising 5 classification tasks (AG News, Amazon, Yelp, DBpedia, Yahoo). These ablations are designed to evaluate: (1) the importance of accurate effective rank estimation for singular vector selection, and (2) the necessity of constraining updates to remain within the low-rank subspace via projection.

\textbf{(1) Impact of Inaccurate Effective Rank Estimation:}  
Our method relies on computing an effective rank per matrix based on input-output activation similarity, which informs the threshold for partitioning singular vectors into high- and low-rank subspaces. To test the importance of this estimation, we reduce both the minimum and target retention ratios (mrr and trr) to half their original values. This results in more aggressive fine-tuning by retaining fewer high singular vectors, thus allocating more of the matrix capacity to learning new tasks. However, this also increases the risk of overwriting components important for previous tasks. As shown in Table~\ref{tab:ablations}, this ablation leads to a substantial performance drop of just over 28 percentage points (from 79.6\% to 51.5\%), emphasizing the importance of accurately estimating the effective rank to ensure that task-relevant subspaces are preserved.

\textbf{(2) Unconstrained Fine-Tuning of Low Singular Vectors:}  
In our method, gradient updates are projected back into the low-rank subspace to prevent interference with high-rank directions. This ablation removes that constraint: we freeze the high singular vectors but allow unconstrained updates to the low singular vectors, meaning that during optimization, updates are not restricted to stay within the initially identified low-rank subspace. This allows the low singular vectors to drift into the space previously occupied by high singular vectors, leading to potential interference and loss of previously acquired knowledge. As expected, this results in catastrophic forgetting, with accuracy dropping from 79.6\% to 31.2\%. In addition, since only the low singular vectors are updated while the high ones are frozen, each new task is forced to be learned in a restricted subspace, limiting the model’s overall expressiveness. Together, these factors result in a $\approx 50$-point accuracy drop, highlighting the necessity of maintaining orthogonality between new task updates and previously learned subspaces.

\begin{table}[h]
\centering
\caption{Ablation results on the LLaMA-2 7B model using the standard 5-task continual learning benchmark.}
\label{tab:ablations}
\begin{tabular}{l c}
\toprule
\textbf{Method} & \textbf{Average Accuracy (\%)} \\
\midrule
Ours (Adaptive SVD) & 79.6 \\
(1) Halved mrr/trr (aggressive effective rank approximation) & 51.5 \\
(2) No projection (unconstrained low-rank updates) & 31.2 \\
\bottomrule
\end{tabular}
\end{table}

\subsection{Implementation Details}
\label{appendix:implementation_details}

We detail the implementation of all experiments presented in this work. Our study utilizes both encoder-decoder and decoder-only language models. For all continual learning experiments—including the 5-task and 15-task benchmarks, as well as the TRACE benchmark—we replicate the task sequences, prompts, and dataset configurations as established in O-LoRA~\cite{wang2024olora} and TRACE~\cite{wang2023trace}.

\paragraph{T5-Large.} Experiments with the T5-Large model were conducted on a single NVIDIA H100 GPU using standard PyTorch training in full precision. We used a constant learning rate of $5 \times 10^{-5}$ with the AdamW optimizer and a total batch size of 8, training for one epoch per task. For each classification dataset, we sampled 1,000 examples per class (where available) to construct balanced training sets, following the protocol established in~\cite{wang2024olora}. All runs were performed with a fixed random seed, and checkpoints were saved after each task for evaluation and reproducibility.

\paragraph{LLaMA-2 7B.} All experiments with the LLaMA-2 7B models were conducted on a server equipped with 8 NVIDIA H100 GPUs, using the DeepSpeed library with Stage 2 optimization. Gradient checkpointing was enabled, and training was performed with a per-GPU batch size of 1 (resulting in an effective batch size of 8). We used the AdamW optimizer with a learning rate of $1 \times 10^{-5}$, weight decay of 0.01, $\beta_1 = 0.9$, $\beta_2 = 0.999$, and $\epsilon = 1 \times 10^{-8}$. All continual learning runs were trained for one epoch per task. After backpropagation, projection steps were applied to the gradients to constrain updates within the designated low-rank subspaces.

Our SVD configuration was automatically generated by analyzing specific matrices in each transformer block—namely, \texttt{q\_proj}, \texttt{k\_proj}, \texttt{v\_proj}, \texttt{o\_proj}, \texttt{gate\_proj}, \texttt{up\_proj}, and \texttt{down\_proj}. Among the various strategies we explored for determining which singular vectors to retain, we found empirically that two approaches consistently performed best. The first allocates a fixed budget by freezing the top $\dfrac{i - 1}{n}$ fraction of singular vectors for task $i$ in an $n$-task sequence. The second uses adaptive rank selection based on layer importance scores, as described in Section~\ref{subsec:adaptive_rank}, where the number of retained singular vectors per layer is computed using the normalized importance $I^{(l)}$ from Section~\ref{subsec:layer_importance}. For this method, we empirically set $\mathrm{mrr}=0.1$ and $\mathrm{trr}=0.8$, which were found to yield consistently strong performance. The remaining components were fine-tuned using projected gradient descent within the low-rank subspace.

\paragraph{Datasets, Task Sequences, and Instructions.} Across all three experimental settings—the 5-task standard CL benchmark, the 15-task longer sequence benchmark, and the 8-task TRACE benchmark—we strictly adhered to the original configurations of O-LoRA~\cite{wang2024olora} and TRACE~\cite{wang2023trace}. This included using the same datasets, task instructions for prompting models during classification and generation, and identical training and validation sample counts and label distributions per task. Task sequences were replicated exactly to ensure consistency across evaluations and facilitate fair comparisons.

\end{document}